\documentclass[sigconf, screen]{acmart}
\AtBeginDocument{%
  }

\setcopyright{acmlicensed}
\copyrightyear{2026}
\acmYear{2026}
\setcopyright{cc}
\setcctype{by}
\acmConference[MM '26]{Proceedings of the 35th ACM International Conference on Multimedia}{November 10--14, 2026}{Rio de Janeiro, Brazil}
\acmBooktitle{Proceedings of the 35th ACM International Conference on Multimedia (MM '26), November 10--14, 2026, Rio de Janeiro, Brazil}
\acmDOI{10.1145/3767308.3836169}
\acmISBN{979-8-4007-2213-4/2026/11}

\usepackage{multirow}
\usepackage{booktabs}
\usepackage{amsthm}
\usepackage{amsmath}
\usepackage[table]{xcolor}
\usepackage[nameinlink,noabbrev]{cleveref}

\newtheorem{assumption}{Assumption}
\newcommand\blfootnote[1]{%
  \begingroup
  \renewcommand\thefootnote{}\footnote{#1}%
  \addtocounter{footnote}{-1}%
  \endgroup
}

\begin{document}

\title{OPERA: Offline Policy-guided Expert Routing and Adaptation for Universal Biomedical Image Analysis}

\author{Zihan Li}
\affiliation{%
  \institution{University of Washington}
  \city{Seattle}
  \country{USA}
}

\author{Feiyang Liu}
\affiliation{%
  \institution{Delft University of Technology}
  \city{Delft}
  \country{Netherlands}
}

\author{Dandan Shan}
\affiliation{%
  \institution{Xiamen University}
  \city{Xiamen}
  \country{China}
  }

\author{Ruibo Wang}
\affiliation{%
  \institution{Delft University of Technology}
  \city{Delft}
  \country{Netherlands}
}

\author{Qingqi Hong}
\affiliation{%
  \institution{Xiamen University}
  \city{Xiamen}
  \country{China}
}

\renewcommand{\shortauthors}{Li, et al.}

\begin{abstract}
Biomedical image analysis spans diverse modalities and tasks, yet real-world deployment is hindered by severe distribution shifts across scanners, protocols, and patient populations. High-performing models consequently require repeated domain-specific fine-tuning, which is a costly cycle that becomes impractical when labels are scarce or privacy constraints limit data sharing. We propose OPERA (Offline Policy-guided Expert Routing and Adaptation), a multi-agent ensemble framework that addresses this deployment bottleneck by treating expert weight assignment as an offline policy learning problem: a routing policy is learned from a small validation set without gradient updates to any expert agent, then deployed with test-time adaptation to handle distribution shift. OPERA coordinates heterogeneous specialist agents through complementary mechanisms. The expert profiling module learns selection policies offline, enabling informed allocation of expertise. Each agent undergoes confidence calibration through temperature adjustment, ensuring more reliable probabilistic outputs. OPERA also incorporates distribution aware adaptation, where class weights are dynamically adjusted at the batch level using statistics derived from unlabeled test data. Instance level routing assigns each sample to the most suitable expert by leveraging inter model agreement and predictive entropy. We evaluate OPERA on 9 datasets covering fundus photography, chest X-ray, CT, MRI, and multimodal diagnostic benchmarks, comparing against 30+ baselines across classification, segmentation, and multimodal settings. OPERA consistently improves performance and calibration quality, demonstrating that offline policy-guided expert agents coordination is a practical path to deployable biomedical AI without retraining. Code is on GitHub\blfootnote{Corresponding author: Zihan Li and Qingqi Hong. \\Li, Z and Liu, F have the equal contribution.}\footnote{\href{https://github.com/HUANGLIZI/OPERA}{https://github.com/HUANGLIZI/OPERA}}.
\end{abstract}

\begin{CCSXML}
<ccs2012>
 <concept>
  <concept_id>00000000.0000000.0000000</concept_id>
  <concept_desc>Do Not Use This Code, Generate the Correct Terms for Your Paper</concept_desc>
  <concept_significance>500</concept_significance>
 </concept>
 <concept>
  <concept_id>00000000.00000000.00000000</concept_id>
  <concept_desc>Do Not Use This Code, Generate the Correct Terms for Your Paper</concept_desc>
  <concept_significance>300</concept_significance>
 </concept>
 <concept>
  <concept_id>00000000.00000000.00000000</concept_id>
  <concept_desc>Do Not Use This Code, Generate the Correct Terms for Your Paper</concept_desc>
  <concept_significance>100</concept_significance>
 </concept>
 <concept>
  <concept_id>00000000.00000000.00000000</concept_id>
  <concept_desc>Do Not Use This Code, Generate the Correct Terms for Your Paper</concept_desc>
  <concept_significance>100</concept_significance>
 </concept>
</ccs2012>
\end{CCSXML}

\ccsdesc[300]{Applied computing~Life and medical sciences}

\keywords{Biomedical image analysis, Multi-agent mechanism}

%% Use this when submitting an article to a sponsored event. You'll
%% receive a unique submission ID from the organizers
%% of the event, and this ID should be used as the parameter to this command.


\maketitle

\section{Introduction}
The medical AI community faces a critical deployment paradox: while specialized models achieve remarkable performance on specific tasks, their brittleness under distribution shift necessitates costly retraining for each new clinical environment. Meanwhile, foundation models promise generalization but require massive computational resources and organized data inaccessible to most healthcare institutions. This raises a fundamental question: is there a third path that combines the strengths of both paradigms? In practice, models that perform well on one dataset or modality often degrade when transferred to a new clinical domain, motivating the pursuit of \emph{universal} biomedical imaging models that can be generalizable.

\begin{figure}[!t]
  \begin{center}
    \centerline{\includegraphics[width=\columnwidth]{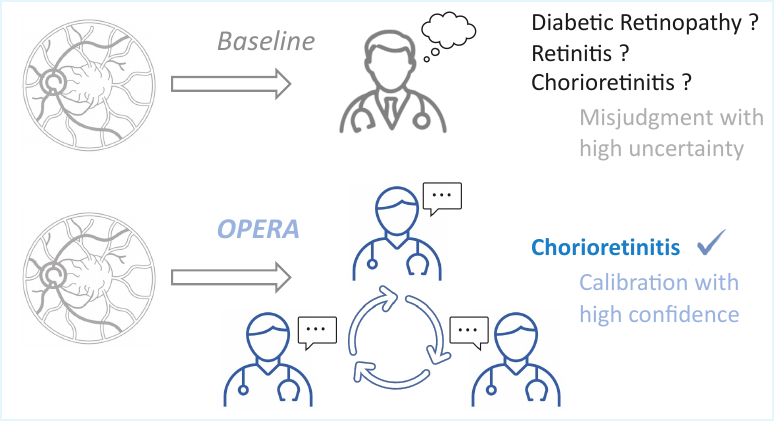}}
    \vspace{-3mm}
    \caption{Motivation of OPERA. \textbf{Top:} a baseline model yields ambiguous predictions for a fundus image (e.g., diabetic retinopathy/retinitis/chorioretinitis), leading to misjudgment with high uncertainty. \textbf{Bottom:} OPERA aggregates multiple expert agents with frozen policies and calibrates their outputs, producing a confident and correct diagnosis (e.g., chorioretinitis) through inference-time ensemble fusion.}
    \label{fig_motivation}
  \end{center}
  \vspace{-8mm}
\end{figure}

Recent foundation and vision-language models provide transferable representations for medical imaging \cite{moor2023, zhou2023foundation, yao2025}, yet they typically depend on large-scale pretraining and/or task-specific fine-tuning to close domain gaps, which is costly and sometimes infeasible when target-domain annotations are limited or unavailable. Similarly, general-purpose segmentation models such as SAM \cite{kirillov2023} show limited zero-shot performance on medical images \cite{mazurowski2023segment}, and adapting them commonly requires additional labeled data and optimization \cite{shi2023}. Ensemble learning offers a complementary route to robustness: combining diverse predictors can reduce correlated errors and improve reliability \cite{Kuncheva2003, Codella2016, Ju2017}. However, standard ensembles are often expensive to train and deploy, frequently requiring joint optimization, stacking, or iterative fine-tuning, and they increase inference cost and latency \cite{Noothout2022, Farhadi2025}. This challenge also mirrors the core problem of offline reinforcement learning: \emph{how to learn a good policy from a fixed historical dataset without environment interaction?} We leverage this correspondence to motivate our offline policy learning formulation and test-time adaptive routing.

To this end, We propose \textbf{OPERA}, a multi-agent expert coordination framework that frames ensemble weight learning as offline policy optimization. The framework comprises expert agents and the coordinator that assigns routing weights through (i) offline policy learning on a validation dataset, (ii) confidence calibration, and (iii) test-time distribution-aware adaptation without backpropagation during deployment. Our main contributions are summarized as follows: (1) we introduce OPERA, the first agentic vision ensemble framework with offline policy learning; (2) we propose a hierarchical inference-time adaptation policy (calibration and dynamic weighting) to improve robustness under distribution shift; and (3) we demonstrate consistent gains of multi-agent coordination generalizes across different modalities without any retraining.

\section{Related Works}

\subsection{Ensemble Learning}
Ensemble learning has proven to be a powerful strategy in computer vision, where combining multiple models' predictions often yields better accuracy and reliability than any single model \cite{Codella2016}. Numerous studies consistently report that ensembles outperform individual networks on diverse medical imaging benchmarks \cite{Ju2017, kong2025mastering}. Effective ensembles depend on the diversity of their constituent models. If the models make mistakes that are not strongly correlated, the ensemble is better positioned to correct them \cite{Kuncheva2003, Ju2017}. While ensemble methods consistently boost accuracy in evaluations, they also introduce notable practical limitations. Running multiple models in parallel substantially increases computational cost, memory consumption, and system latency, which is a major challenge for time-sensitive and resource-constrained clinical deployment scenarios \cite{Noothout2022, liu2019accurate, cao2024predictive}. In addition, many existing ensemble strategies require joint retraining, complex stacking procedures, further raising the barrier to real-world adoption \cite{Farhadi2025, li2022semi}. In contrast to conventional ensemble methods that rely on computationally expensive joint retraining, iterative fine-tuning, or complex meta-learning schemes \cite{Ju2017, shi2003study, gordji2017orthogonal}, our approach introduces an offline-calibrated zero-retraining framework for universal medical image analysis. We integrate diverse models exclusively during the inference phase, utilizing a dynamic weighting mechanism that synthesizes model outputs without requiring backpropagation and parameter updates. Our proposed OPERA preserves the accuracy and robustness advantages of ensembling while eliminating the substantial computational overhead with ensemble methods.

\subsection{Offline Policy Learning}

Offline policy learning trains policies entirely from fixed historical datasets, eliminating the need for active environmental interaction~\cite{levine2020offline,prudencio2023survey}. This paradigm is particularly suited to medical imaging, where online exploration entails prohibitive safety and regulatory risks. A central challenge is distributional shift: when a learned policy proposes actions outside the behavioral training distribution, standard temporal-difference learning overestimates Q-values on out-of-distribution state--action pairs, leading to policy degradation. Contemporary methods address this through explicit behavioral constraints or pessimistic value estimation, which assigns conservatively low returns to unseen transitions~\cite{levine2020offline}. In OPERA, the Expert Policy Memory (EPM) module instantiates these principles by encoding expert radiotherapy trajectories with distributional conservatism, preventing overconfident extrapolation to under-represented imaging modalities. Recent work on offline-to-online RL demonstrates that offline-pretrained policies can be efficiently refined at deployment time~\cite{lee2022offline2online}. Offline pretraining substantially compresses the online exploration horizon, allowing lightweight adaptation to address residual distributional gaps rather than relearning from scratch. OPERA operationalizes this via Dynamic Anchor Adaptation (DAA) and Intelligent Experience Replay (IER), which perform test-time policy refinement as site-specific imaging characteristics are encountered. The routing logic of IER draws on the Mixture-of-Experts (MoE) paradigm~\cite{shazeer2017moe,fedus2022switch}, wherein a learned sparse gating function selectively activates specialized sub-modules per input, providing both parameter efficiency and expert specialization. IER adapts this mechanism to experience replay: incoming deployment trajectories are routed to modality-specific replay buffers, ensuring adaptation signals remain contextually coherent. The Mixture-of-Agents frameworks~\cite{wang2024moa,wang2026any2any} also motivate OPERA's multi-specialist ensemble design, wherein diverse modality-specific agents refine predictions through iterative aggregation.

\subsection{Biomedical Image Analysis}
Large foundation models are increasingly being adopted in medical imaging because their learned representations transfer well across tasks~\cite{moor2023}. However, adapting these models to new domains typically requires extensive self-supervised pretraining on massive datasets and computationally expensive fine-tuning~\cite{zhou2023foundation}. While Parameter-Efficient Fine-Tuning (PEFT) methods such as LoRA and adapters have emerged to reduce these adaptation costs~\cite{hu2022lora, houlsby2019}, they still require target-domain annotations and model updates prior to deployment, limiting flexibility when encountering truly unseen clinical distributions. Test-Time Adaptation (TTA) offers an alternative by adapting models during inference using only unlabeled test data~\cite{liang2024comprehensive}, yet existing TTA approaches often focus on single-model scenarios and may struggle with the diverse distribution shifts in anatomy, imaging protocols, and noise characteristics inherent to medical imaging. In addition, the general-purpose Segment Anything Model (SAM)~\cite{kirillov2023} exhibits limited zero-shot performance on medical images~\cite{mazurowski2023segment, li2026large}, and while task-specific fine-tuning can alleviate this gap~\cite{shi2023, ando2011predictive, li2025boosting}, it necessitates aligned annotations and retraining for each new domain. Building on these insights, we propose OPERA, a framework that synergizes ensemble learning with specialized expert agents for universal medical image analysis. OPERA incorporates test-time adaptation mechanisms that dynamically adjust model contributions during inference without requiring labeled data or model retraining, distinguishing our approach from fine-tuning and existing TTA methods.

\begin{figure*}[t]
  \begin{center}
    \centerline{\includegraphics[width=0.95\textwidth]{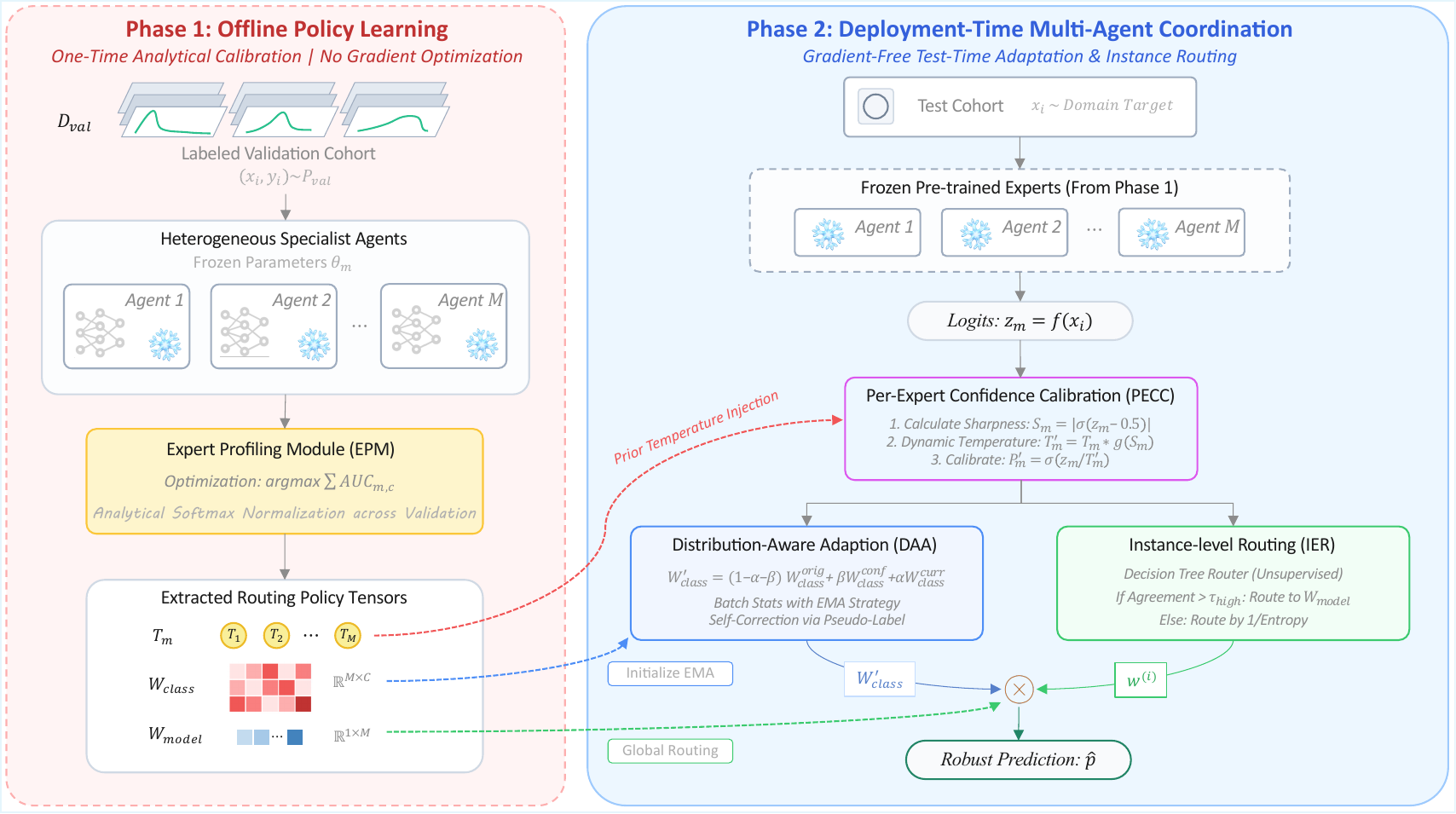}}
    \vspace{-3mm}
    \caption{Overview of the OPERA framework. In the Offline Stage, the Expert Profiling Module (EPM) extracts initial model weights ($w_{model}$), class-level weights ($W_{class}$), and temperature parameters ($T_m$) from validation data. During Deployment-Time, frozen expert predictions undergo Per-Expert Confidence Calibration (PECC) to adjust sharpness and output calibrated probabilities ($p_{i,c}^{\prime(m)}$). The ensemble is then adapted via the Dynamic Fusion module: Distribution-Aware Adaptation (DAA) refines class-level weights ($W_{class}^{\prime}$) dynamically using batch statistics. Instance-level Expert Routing (IER) generates sample-specific weights ($w^{(i)}$) leveraging inter-model agreement and uncertainty.}
    \label{fig_pipeline}
  \end{center}
  \vspace{-7mm}
\end{figure*}

\section{Methods}
OPERA treats each pre-trained expert as a specialist agent and the ensemble fusion module as an orchestrator that learns offline which agents to trust for each disease class and each input instance. The framework integrates three complementary medical imaging models for different tasks. Rather than uniformly averaging predictions, OPERA employs a multi-level weighting strategy that dynamically adjusts the contribution of each expert agent based on its individual strengths and the characteristics of input data. Here, $M$ denotes the number of expert agents, $N$ the number of samples, $C$ the number of disease classes, and $B$ the batch size. For model $m \in \{1, \ldots, M\}$, we denote its probability predictions as $\mathbf{p}^{(m)} \in [0,1]^{N \times C}$, where $p^{(m)}_{i,c}$ represents the predicted probability for sample $i$ and class $c$. We also frame the segmentation task as a pixel-wise classification task in our study.

\subsection{Expert Profiling Module (EPM)}

Expert Profiling Module establishes initial ensemble weights by analyzing model performance on the offline dataset for policy learning with available ground truth labels $\mathbf{y} \in \{0,1\}^{N \times C}$. Formally, EPM solves an offline policy learning problem: given a fixed offline dataset $D_{val} = {(x_i, y_i)}$, learn a routing policy. Crucially, this optimization requires only forward passes through frozen expert models, with no gradient updates, satisfying the offline RL property of policy improvement without environment interaction. For each model-class pair $(m,c)$, we compute AUC for the reward-driven policy optimization $\text{AUC}_{m,c} = \text{AUROC}(\mathbf{y}_{\cdot,c}, \mathbf{p}^{(m)}_{\cdot,c})$, where $\mathbf{y}_{\cdot,c}$ and $\mathbf{p}^{(m)}_{\cdot,c}$ denote the ground truth labels and predictions for class $c$ across all validation samples. The performance matrix $\mathbf{A} \in \mathbb{R}^{M \times C}$ with entries $A_{m,c} = \text{AUC}_{m,c}$ captures the expertise distribution across models and classes.  From this performance matrix, we derive model-level weights $\mathbf{w}_{\text{model}} \in \mathbb{R}^M$ that reflect overall model quality. First, we compute the average performance for each model across all classes and $M$ is set to 3 in our study: $\bar{A}_m = \frac{1}{C}\sum_{c=1}^{C} A_{m,c}$.

These averages are then normalized using temperature-scaled softmax to obtain model-level weights:
\begin{equation}
w_{\text{model},m} = \frac{\exp(\tau \bar{A}_m)}{\sum_{m'=1}^{M} \exp(\tau \bar{A}_{m'})},
\end{equation}
where $\tau$ is a temperature parameter that controls weight concentration. We set $\tau=10.0$ to balance specialization and diversity, ensuring that better-performing models receive higher weights while maintaining ensemble diversity.

To capture model-specific expertise at the disease level, we introduce class-level attention weights $\mathbf{W}_{\text{class}} \in \mathbb{R}^{M \times C}$. For each class $c$, we normalize its performance of all models:

\begin{equation}
W_{\text{class},m,c} = \frac{\exp(\tau A_{m,c})}{\sum_{m'=1}^{M} \exp(\tau A_{m',c})}.
\label{eq4}
\end{equation}

This formulation ensures that $\sum_{m=1}^{M} W_{\text{class},m,c} = 1$ for each class, allowing models with higher AUC for a specific disease to receive greater weight when predicting that condition. Additionally, we compute temperature calibration parameters for each model based on their validation performance. The complete set of tuned weights, including all model-level weights set $\mathbf{w}_{\text{model}}$, class-level weights set $\mathbf{W}_{\text{class}}$, and initial temperatures $\{T_m\}_{m=1}^M$, serves as the learned routing policy  for subsequent deployment-time adaptation.

\subsection{Per-Expert Confidence Calibration (PECC)}

Per-Expert Confidence Calibration addresses prediction calibration by automatically adjusting the sharpness of each model's probability outputs based on their prediction characteristics. This optimization operates independently for each model without requiring labeled data, making it applicable during both validation and test phases. For each model $m$, we assess prediction sharpness by measuring how far predictions deviate from maximum uncertainty across a batch:
\begin{equation}
s_m = \frac{1}{BC}\sum_{i=1}^{B}\sum_{c=1}^{C} |p^{(m)}_{i,c} - 0.5|,
\end{equation}
where higher values indicate more decisive predictions. Based on this metric, the temperature parameter $T_m$ for model $m$ is selected as 1.5 ($s_m$ > 0.4), 0.7 ($s_m$ < 0.1) or 1.0 according to the value of $s_m$.

The temperature scaling is applied through a logit transformation. Given predicted probabilities $\mathbf{p}^{(m)}$, we first convert them to logits using the inverse sigmoid function:
\begin{equation}
z^{(m)}_{i,c} = \log\left(\frac{p^{(m)}_{i,c} + \epsilon}{1 - p^{(m)}_{i,c} + \epsilon}\right),
\end{equation}
where $\epsilon = 10^{-7}$ prevents numerical instability at boundaries. The logits are then scaled by the temperature and converted back to calibrated probabilities:
\begin{equation}
p'^{(m)}_{i,c} = \sigma\left(\frac{z^{(m)}_{i,c}}{T_m}\right) = \frac{1}{1 + \exp\left(-\frac{z^{(m)}_{i,c}}{T_m}\right)}.
\end{equation}
Temperature scaling preserves the rank ordering of predictions while adjusting their spread, ensuring that both individual model contributions and their confidence levels are appropriately calibrated for ensemble combination.

\subsection{Distribution-Aware Adaptation (DAA)}

DAA adapts the class-level attention weights dynamically during inference by analyzing prediction confidence patterns on the test set, without requiring ground truth labels. This online policy correction under shift allows the model to respond to distribution shifts or domain-specific characteristics in test data while keeping model-level weights fixed to preserve the overall quality hierarchy. For each test batch, we compute class-level confidence scores that reflect how decisive each model's predictions are. Given a model $m$ and class $c$, the confidence for a batch of size $B$ is measured as the average deviation from maximum uncertainty:
\begin{equation}
\text{conf}_{m,c}^{(k)} = \frac{1}{B}\sum_{i=1}^{B} |p'^{(m)}_{i,c} - 0.5|,
\label{eq9}
\end{equation}
where $k$ denotes the batch index and higher confidence indicates predictions closer to zero or one rather than ambiguous values near 0.5. These confidence scores are accumulated across $K$ processed batches to compute running averages:
$\bar{\text{conf}}_{m,c} = \frac{1}{K}\sum_{k=1}^{K} \text{conf}^{(k)}_{m,c}$.
For every 100 batches, we update the class-level attention weights set $\mathbf{W}'_{\text{class}} = \{w'_{\text{class},m}\}_{m=1}^{M}$ by blending the original validation-tuned weights with test-time confidence statistics based on the previous $k$ batches. Let $\mathbf{W}^{\text{orig}}_{\text{class}}$ denote the original validation-tuned weights, $\mathbf{W}^{\text{curr}}_{\text{class}}$ the current adapted weights, and $\mathbf{W}^{\text{conf}}_{\text{class}}$ the confidence-based weights. The update rule is:
\begin{equation}
\mathbf{W}'_{\text{class}} = (1 - \alpha - \beta) \mathbf{W}^{\text{orig}}_{\text{class}} + \beta \mathbf{W}^{\text{conf}}_{\text{class}} + \alpha \mathbf{W}^{\text{curr}}_{\text{class}},
\label{eq11}
\end{equation}
where the confidence-based weights set $\mathbf{W}^{\text{conf}}_{\text{class}}$ are computed by normalizing the accumulated confidence scores:
\begin{equation}
W^{\text{conf}}_{\text{class},m,c} = \frac{\exp(\bar{\text{conf}}_{m,c})}{\sum_{m'=1}^{M} \exp(\bar{\text{conf}}_{m',c})}.
\end{equation}
We set $\alpha = 0.1$ and $\beta = 0.1$ to ensure gradual adaptation that preserves validation-based knowledge while incorporating test-time observations. This adaptive strategy enables the ensemble to adjust to class-specific distribution shifts without requiring labeled data or model retraining.

\subsection{Instance-level Expert Routing (IER)}

Instance-level Expert Routing introduces sample-specific weight adjustments that account for the varying reliability of different models across individual test instances. Unlike the previously described methods that operate at the model or class level, IER produces instance-level weights $\mathbf{w}^{(i)} \in \mathbb{R}^M$ for each sample $i$ (each pixel in segmentation), recognizing that model performance can vary substantially depending on input characteristics. The weighting strategy relies on two unsupervised metrics computed for each sample: inter-model agreement and per-model prediction uncertainty. Agreement between models is measured using pairwise cosine similarity of predictions:
\begin{equation}
\text{sim}(m, m', i) = \frac{\mathbf{p'}^{(m)}_{i,\cdot} \cdot \mathbf{p'}^{(m')}_{i,\cdot}}{\|\mathbf{p'}^{(m)}_{i,\cdot}\| \|\mathbf{p'}^{(m')}_{i,\cdot}\|},
\end{equation}
where $\mathbf{p'}^{(m)}_{i,\cdot}$ denotes the calibrated prediction vector for sample $i$ across all classes. The overall agreement score for sample $i$ is the average similarity across all model pairs:
\begin{equation}
a_i = \frac{2}{M(M-1)}\sum_{m=1}^{M-1}\sum_{m'=m+1}^{M} \text{sim}(m, m', i).
\label{eq14}
\end{equation}
Prediction uncertainty for each model is quantified via averaged binary entropy $H(p)=-p\log p-(1-p)\log(1-p)$:
\begin{equation}
h_{i,m}=\frac{1}{C}\sum_{c=1}^C H\left(p'^{(m)}_{i,c}\right).
\end{equation}
% Based on these metrics, we apply three conditional weighting strategies. For high-agreement cases where $a_i > 0.90$ and the variance across model predictions $\sigma^2_i < 0.03$, we use the base model-level weights: $\mathbf{w}^{(i)} = \mathbf{w}_{\text{model}}$. For high-disagreement cases where $a_i < 0.60$ or $\sigma^2_i > 0.12$, we weight models inversely proportional to their entropy:
Based on these metrics, we apply three conditional weighting strategies. We first compute the variance across model predictions for sample $i$ as $\sigma^2_i = \frac{1}{MC} \sum_{m=1}^{M} \sum_{c=1}^{C} \left( p'^{(m)}_{i,c} - \bar{p}'_{i,c} \right)^2$, where $\bar{p}'_{i,c}$ is the mean prediction across all $M$ models for sample $i$ and class $c$. This variance quantifies the degree of disagreement among expert agents at the prediction level. For high-agreement cases where $a_i > 0.90$ and $\sigma^2_i < 0.03$, we use base model-level weights: $\mathbf{w}^{(i)} = \mathbf{w}_{\text{model}}$. For high-disagreement cases where $a_i < 0.60$ or $\sigma^2_i > 0.12$, we weight models inversely proportional to their entropy:
\begin{equation}
w^{(i)}_m \propto \exp\left(\frac{\gamma}{h_{i,m} + \epsilon}\right),
\label{eq16}
\end{equation}
where $\gamma = 2.0$ controls differentiation strength. For intermediate cases, we weight by prediction extremity $e_{i,m} = \frac{1}{C}\sum_{c=1}^{C}|p^{(m)}_{i,c} - 0.5|$ with base weights ($\mu$=3.0, $\zeta$ = 0.7):
\begin{equation}
\mathbf{w}^{(i)} = \zeta \cdot \text{softmax}(\mu \cdot \mathbf{e}_i) + (1-\zeta) \cdot \mathbf{w}_{\text{model}}.
\end{equation}

The final ensemble prediction for sample $i$ combines all weighting levels through element-wise multiplication and normalization. The complete weight matrix for sample $i$ is:
\begin{equation}
W_{\text{final},i,m,c} = w_{\text{model},m} \cdot {w}'_{\text{class,m,c}} \cdot w^{(i)}_m,
\label{eq18}
\end{equation}
which is normalized across models for each class:
\begin{equation}
W'_{\text{final},i,m,c} = W_{\text{final},i,m,c} / \sum_{m'=1}^{M} W_{\text{final},i,m',c}
\end{equation}
The final ensemble prediction is then computed as:
\begin{equation}
\hat{p}_{i,c} = \sum_{m=1}^{M} W'_{\text{final},i,m,c} \cdot p'^{(m)}_{i,c},
\end{equation}
where $p'^{(m)}_{i,c}$ denotes the temperature-calibrated prediction from PECC. This hierarchical weighting structure enables the ensemble to leverage complementary model strengths at multiple scales: overall model quality, disease-specific and sample-specific expertise.

\section{Experiments}
\paragraph{Implementation Details} All our experiments are conducted with NVIDIA RTX A6000 GPUs. We use AUC (Area Under the Curve), ACC (Accuracy), and mAP (mean Average Precision) to compare the performance of classification task. As for the multi-label setting, we report the category-wise average accuracy as ACC. Dice coefficient (Dice) and Jaccard coefficient (Jaccard) are used in the segmentation task. For fair comparison, all baselines are fine-tuned on the same training data. More details are in the appendix.

\begin{table}[!t]
\LARGE
    \centering
    \caption{The performance evaluation on RFMiD. \textbf{Bold} indicates best performance and \underline{underline} shows second-best. CLIP-\textit{X} denotes using \textit{X} as vision encoder with CLIP.} % 7680
    \vspace{-4mm}
    \label{tab:performance_rfmid}
    \resizebox{\linewidth}{!}{
        \begin{tabular}{l|ccc|ccc}
        \toprule
        \multirow{2}{*}{Method} & \multicolumn{3}{c|}{\textbf{Linear Probe (\%)}} & \multicolumn{3}{c}{\textbf{Fully Fine-tune (\%)}} \\
                & AUC   & ACC   & mAP   & AUC   & ACC   & mAP \\ 
        \hline
        KgCoOp \textcolor{gray}{(CVPR-23)} & 50.82 & 92.21 & 7.96 & 70.36 & 92.28 & 21.49\\
        RETFound \textcolor{gray}{(Nature-23)} & 60.16 & 92.55 & 16.37 & 84.62 & 93.48 & 50.21\\
        KeepFIT \textcolor{gray}{(MICCAI-24)} & 51.48 & 92.19 & 10.24 & 81.52 & 93.09 & 42.39\\
        RET-CLIP \textcolor{gray}{(MICCAI-24)} & 58.94 & 92.46 & 16.02 & \underline{86.12} & \underline{93.92} & \underline{51.27}\\
        UniMed-CLIP \textcolor{gray}{(arXiv-24)} & 53.44 & 92.25 & 13.68 & 72.59 & 92.57 & 22.54 \\
        VisionFM \textcolor{gray}{(NEJM AI-24)}  & \underline{63.38} & \underline{92.59} & \underline{17.84} &  82.78 & 93.17 & 48.92 \\
        FLAIR \textcolor{gray}{(MedIA-25)} & 56.11 & 92.38 & 14.85 & 75.43 & 92.62 & 23.24\\
        M3DT \textcolor{gray}{(ICML-25)} & 61.07 & 92.50 & 16.59 & 84.90 & 93.52 & 50.48\\
        TCA \textcolor{gray}{(CVPR-25)} & 58.83 & 92.41 & 15.66 & 83.69 & 93.22 & 50.06\\\hline
        \textbf{OPERA (Ours)}& \textbf{73.25} & \textbf{93.06} & \textbf{23.82} & \textbf{89.47} & \textbf{94.64} & \textbf{58.93} \\\rowcolor{gray!20}
        ~~~$\mathbb{E}_1$: CLIP-ViT & 44.66 & 92.53 & 7.28 & 65.10 & 92.86 & 17.31 \\\rowcolor{gray!20}
        ~~~$\mathbb{E}_2$: CLIP-ResNet50 & 48.14 & 91.98 & 8.09 & 64.84 & 92.29 & 16.87 \\\rowcolor{gray!20}
         ~~~$\mathbb{E}_3$: CLIP-DenseNet121 & 50.29 & 92.20 & 10.23 & 69.58 & 92.51 & 20.54 \\
        \bottomrule
        \end{tabular}}
    \vspace{-4mm}
\end{table}

\begin{table}[!t]
\LARGE
    \centering
    \caption{The performance evaluation on OIA-DDR. \textbf{Bold} indicates best performance and \underline{underline} shows second-best. CLIP-\textit{X} denotes using \textit{X} as the vision encoder with CLIP framework. OPERA* denotes the model variant without using the labeled data to initialize model-level weights.} %3759
    \vspace{-4mm}
    \label{tab:performance_oiaddr}
    \resizebox{\linewidth}{!}{
        \begin{tabular}{l|ccc|ccc}
        \toprule
        \multirow{2}{*}{Method} & \multicolumn{3}{c|}{\textbf{Linear Probe (\%)}} & \multicolumn{3}{c}{\textbf{Fully Fine-tune (\%)}} \\
                & AUC   & ACC   & mAP   & AUC   & ACC   & mAP \\ 
        \hline
        KgCoOp \textcolor{gray}{(CVPR-23)} & 72.09 & 53.68 & 37.87 & 80.39 & 65.47 & 45.92 \\
        RETFound \textcolor{gray}{(Nature-23)} & 80.77 & 61.13 & 46.94 & \underline{85.96} & \underline{72.12} & \underline{52.21} \\
        KeepFIT \textcolor{gray}{(MICCAI-24)}  & 72.68 & 54.46 & 38.72 & 79.42 & 64.03 & 44.53 \\
        RET-CLIP \textcolor{gray}{(MICCAI-24)} & 77.43 & 58.18 & 43.31 & 83.68 & 68.82 & 48.05 \\
        UniMed-CLIP \textcolor{gray}{(arXiv-24)} & 74.67 & 56.19 & 40.18 & 81.23 & 66.69 & 46.88 \\
        VisionFM \textcolor{gray}{(NEJM AI-24)}  & 78.85 & 59.35 & 44.27 & 84.25 & 69.46 & 48.77 \\
        FLAIR \textcolor{gray}{(MedIA-25)} & \underline{82.48} & \underline{63.36} & \underline{48.09} & 85.54 & 70.63 & 50.27 \\
        TCA \textcolor{gray}{(CVPR-25)} & 79.67 & 59.80 & 45.19 & 85.81 & 71.53 &  50.92\\\hline
        \textbf{OPERA (Ours)} & \textbf{84.95} & \textbf{69.57} & \textbf{49.81} & \textbf{88.06} & \textbf{72.89} & \textbf{55.90} \\\rowcolor{gray!20}
        ~~~$\mathbb{E}_1$: CLIP-ViT& 73.30 & 55.41 & 39.21 & 85.29 & 71.75 & 49.91 \\\rowcolor{gray!20}
        ~~~$\mathbb{E}_2$: CLIP-ResNet50 & 71.69 & 54.91 & 37.52 & 80.53 & 69.78 & 46.22 \\\rowcolor{gray!20}
         ~~~$\mathbb{E}_3$: CLIP-DenseNet121 & 72.73 & 55.20 & 38.19 & 81.46 & 69.94 & 47.09 \\
         OPERA* (Ours) & 83.18 & 67.07 & 49.33 & 86.30 & 72.52 & 54.29 \\
        \bottomrule
        \end{tabular}}
    \vspace{-6mm}
\end{table}

\paragraph{Datasets and Baselines.} We utilize RFMiD~\cite{pachade2021retinal}, OIA-DDR~\cite{li2019diagnostic}, Chest X-Ray14~\cite{wang2017chestx}, OrganSMNIST~\cite{yang2023medmnist}, QaTa-COV19~\cite{degerli2022osegnet}, MosMedData+~\cite{morozov2020mosmeddata,hofmanninger2020automatic}, LA-MRI~\cite{xiong2021global}, Pancreas-CT~\cite{clark2013cancer} and Multiple-choice Evaluation Benchmark~\cite{li2025visionunite} as our evaluation datasets to compare OPERA with other baseline methods in diverse tasks. The QaTa-COV19 and MosMedData+ datasets follow the standard data split in the work~\cite{li2023lvit}. We compare proposed OPERA including fine-tuned models (ViT \cite{dosovitskiy2020image,yao2025}, ResNet50 \cite{he2016deep,wang2022multi}, DenseNet121 \cite{huang2017densely,xiao2023delving}) against 15 state-of-the-art baselines \cite{yao2023visual, zhou2023foundation, wu2024mm, du2024ret, khattak2024unimed, qiu2024development, silva2025foundation, xie2024unimiss+, perez2025exploring, yang2025chest, ma2025fully, wu2025large, luo2025ensemble, ni2025maintaining, kong2025mastering} on the image classification of fundus or X-ray. We utilize CLIP \cite{radford2021learning} as default vision-language pretraining. For the segmentation task, we compare proposed OPERA including TransUNet \cite{chen2024transunet}, U-Net \cite{ronneberger2015u}, V-Net \cite{milletari2016v} against 13 state-of-the-art baselines \cite{isensee2021nnu, xie2024unimiss+, perez2025exploring, ma2024segment, shan2025stpnet, wu2025large, wu2022mutual, basak2023pseudo, adiga2024anatomically, ma2024constructing, qi2024gradient, li2025scale, ni2025maintaining}.
\begin{figure*}[!ht]
  \begin{center}
    \centerline{\includegraphics[width=0.9\textwidth]{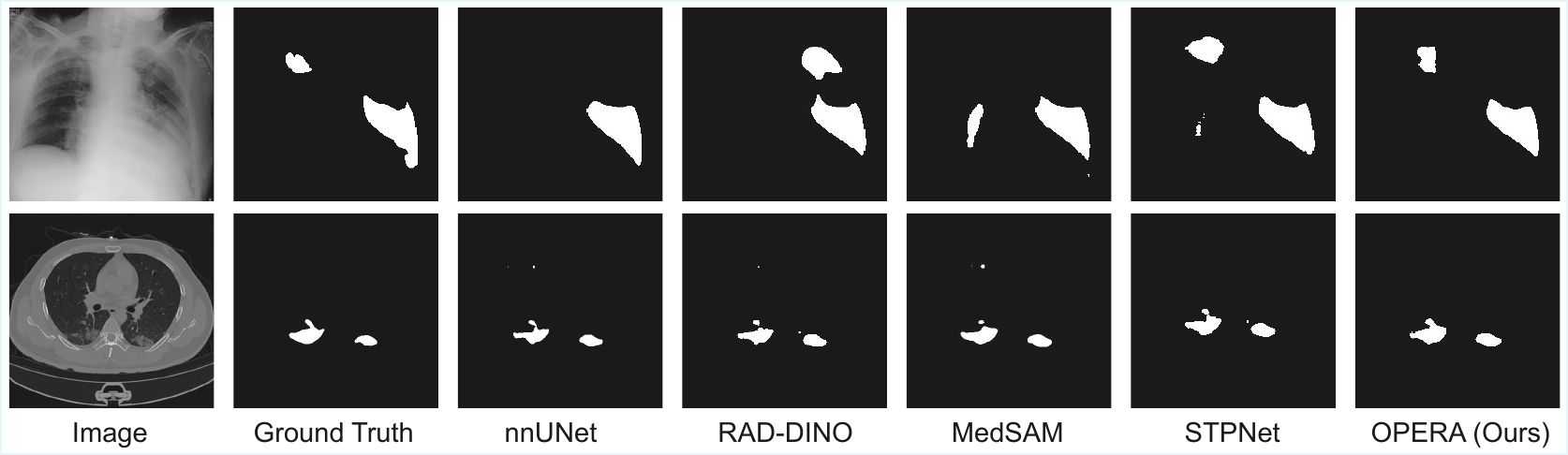}}
    \vspace{-4mm}
    \caption{Qualitative results on COVID-19 datasets. \textbf{Top:} COVID-Xray (QaTa-COV19). \textbf{Bottom:} COVID-CT (MosMedData+). OPERA yields cleaner lesion boundaries and better coverage of infection regions for low-contrast and irregular shapes.}
    \label{fig_results}
  \end{center}
  \vspace{-8mm}
\end{figure*}

\begin{figure*}[!ht]
  \begin{center}
    \centerline{\includegraphics[width=0.9\textwidth]{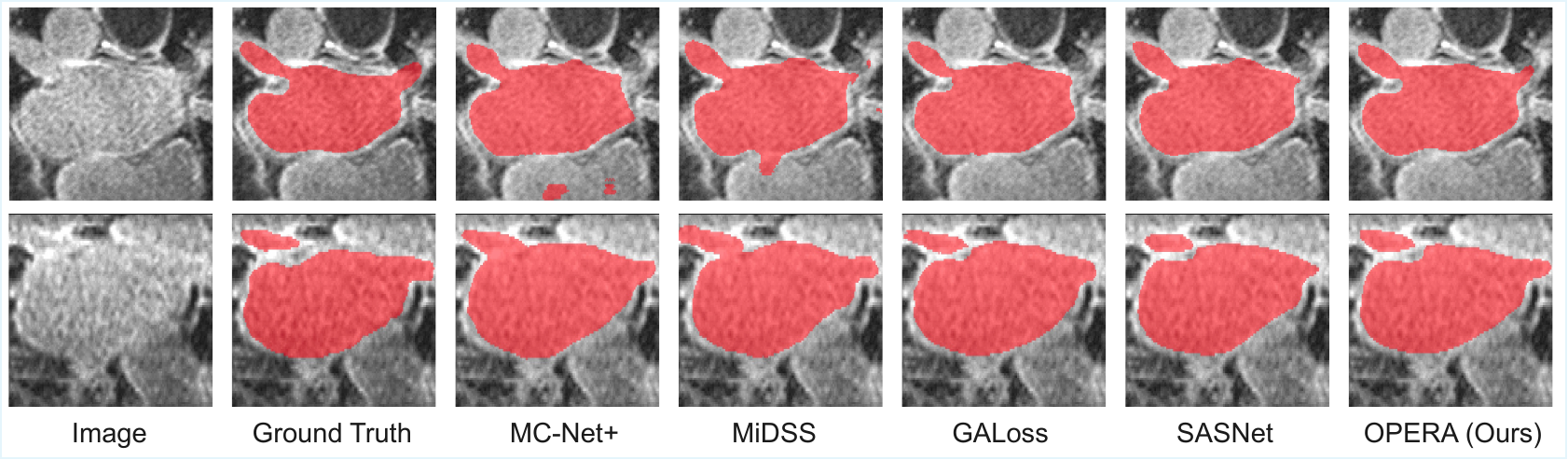}}
    \vspace{-4mm}
    \caption{Qualitative segmentation results on the LA-MRI dataset with 20\% labeled data.}
    \label{sup_results3}
  \end{center}
  \vspace{-8mm}
\end{figure*}

\subsection{Vision-Language Modeling for Image Classification}
We evaluate OPERA under the CLIP-style vision-language modeling setup on two ophthalmic classification benchmarks (RFMiD and OIA-DDR). Table~\ref{tab:performance_rfmid} and Table~\ref{tab:performance_oiaddr} report results under both linear probing and full fine-tuning. Overall, OPERA consistently achieves the top performance across AUC, ACC, and mAP, demonstrating that ensembling heterogeneous CLIP vision encoders yields more robust medical representations. On RFMiD (Table~\ref{tab:performance_rfmid}), OPERA reaches best performance with full fine-tuning, surpassing RET-CLIP by 3.35\% AUC and 7.66\% mAP. On OIA-DDR (Table~\ref{tab:performance_oiaddr}), OPERA also achieves the best results. With full fine-tuning, OPERA beats over the second-best method RETFound by 2.10\% AUC and 3.69\% mAP. The results of OPERA* shows that OPERA degrades gracefully to a labels-free mode while retaining strong performance compared with SOTA test-time adaptation method \cite{ni2025maintaining} and offline RL \cite{kong2025mastering}.

\begin{table}[!t]
\LARGE
    \centering
    \caption{The performance evaluation on Chest X-ray Classification (Chest X-Ray14) and Abdominal CT Classification (OrganSMNIST). \textbf{Bold} indicates best performance and \underline{underline} shows second-best. ELF* denotes our re-implementation of ELF utilizing $\mathbb{E}_1$, $\mathbb{E}_2$ and $\mathbb{E}_3$ models.} %25596, 8827
    \vspace{-4mm}
    \label{tab:performance_XrayCT}
    \resizebox{\linewidth}{!}{
        \begin{tabular}{l|ccc|ccc}
        \toprule
        \multirow{2}{*}{Method} & \multicolumn{3}{c|}{\textbf{Chest X-Ray14 (\%)}} & \multicolumn{3}{c}{\textbf{OrganSMNIST (\%)}} \\
                & AUC   & ACC   & mAP   & AUC   & ACC   & mAP \\ 
        \hline
        UniMiSS+ \textcolor{gray}{(TPAMI-24)} & 81.45 & 92.41 & 27.06 & 97.90 & 80.25 & 75.64\\
        RAD-DINO \textcolor{gray}{(NMI-25)} & 81.37 & 92.36 & 26.92 & 98.05 & 83.21 & 79.01 \\
        CheXFound \textcolor{gray}{(TMI-25)} & 81.63 & 92.45 & 27.11 & 97.79 & 79.26 & 73.88 \\
        Ark+ \textcolor{gray}{(Nature-25)} & \underline{82.42} & \underline{92.62} & 30.11 & 98.16 & 83.03 & 78.54 \\
        VoCo \textcolor{gray}{(TPAMI-25)}  & 82.30 & 92.57 & 30.03 & 98.37 & \underline{83.65} & 79.32 \\
        ELF* \textcolor{gray}{(Arxiv-25)} & 82.34 & 92.62 & \underline{30.24} & \underline{98.41} & 83.58 & \underline{79.44} \\\hline
        \textbf{OPERA (Ours)}& \textbf{83.05} & \textbf{92.64} & \textbf{30.58} & \textbf{98.83} & \textbf{84.31} & \textbf{80.12} \\\rowcolor{gray!20}
        ~~~~~~$\mathbb{E}_1$: ViT & 82.28 & 92.62 & 30.20 & 98.14 & 83.17 & 78.97 \\\rowcolor{gray!20}
        ~~~~~~$\mathbb{E}_2$: ResNet50 & 81.28 & 92.42 & 28.51 & 97.68 & 78.61 & 73.19 \\\rowcolor{gray!20}
         ~~~~~~$\mathbb{E}_3$: DenseNet121 & 81.29 & 92.50 & 27.43 & 97.85 & 82.54 & 77.03 \\
        \bottomrule
        \end{tabular}}
    \vspace{-6mm}
\end{table}

\subsection{Universal Biomedical Image Classification}

We further evaluate OPERA on Chest X-Ray14 for X-ray classification. As shown in Table~\ref{tab:performance_XrayCT}, OPERA achieves the best overall performance with \textbf{83.05\%} AUC, \textbf{92.64\%} ACC, and \textbf{30.58\%} mAP. OPERA improves over Ark+ (82.42\% AUC) and also surpasses the SOTA ensemble method ELF in AUC (83.05\% vs. 82.34\%), indicating better robustness. OrganSMNIST is used for abdominal CT classification. OPERA again performs best (Table~\ref{tab:performance_XrayCT}), reaching \textbf{98.83\%} AUC, \textbf{84.31\%} ACC, and \textbf{80.12\%} mAP, demonstrating strong generalization across imaging modalities and anatomical structures.

\subsection{COVID-19 X-ray and CT Segmentation}

\begin{table}[!t]
\LARGE
    \centering
    \caption{The performance evaluation on COVID-19 X-ray (QaTa-COV19) and COVID-19 CT (MosMedData+) Segmentation. \textbf{Bold} indicates best performance and \underline{underline} shows second-best. OPERA* denotes the model variant without using the labeled data to initialize model-level weights.}
    \vspace{-4mm}
    \label{tab:performance_covid}
    \resizebox{\linewidth}{!}{
        \begin{tabular}{l|cc|cc}
        \toprule
        \multirow{2}{*}{Method} & \multicolumn{2}{c|}{\textbf{QaTa-COV19}} & \multicolumn{2}{c}{\textbf{MosMedData+}} \\
                & Dice (\%) & Jaccard (\%)& Dice (\%)& Jaccard (\%)\\ 
        \hline
        nnUNet \textcolor{gray}{(Nat. Meth-21)} & 80.42 & 70.81 & 72.59 & 60.36 \\
        UniMiSS+ \textcolor{gray}{(TPAMI-24)} & 80.07 & 70.13 & 72.45 & 60.21 \\
        MedSAM \textcolor{gray}{(Nat. Com-24)} & 79.31 & 69.89 & 71.57 & 59.34 \\
        RAD-DINO \textcolor{gray}{(NMI-25)} & 80.58 & 71.38 & 75.13 & 62.55 \\
        TCA \textcolor{gray}{(CVPR-25)} & 80.37 & 71.06 & 75.34 & 62.93 \\
        STPNet \textcolor{gray}{(TIP-25)} & \underline{80.63} & \underline{71.42} & \underline{76.18} & \underline{63.41}\\
        VoCo \textcolor{gray}{(TPAMI-25)} & 80.19 & 70.38 & 73.08 & 60.53 \\\hline
        \textbf{OPERA (Ours)} & \textbf{82.12} & \textbf{73.26} & \textbf{77.03}  & \textbf{64.39} \\\rowcolor{gray!20}
        ~~~~~~$\mathbb{E}_1$: TransUNet & 78.63 & 69.13 & 71.24 & 58.44 \\\rowcolor{gray!20}
        ~~~~~~$\mathbb{E}_2$: U-Net & 79.02 & 69.46 & 64.60 & 50.73 \\\rowcolor{gray!20}
         ~~~~~~$\mathbb{E}_3$: V-Net & 78.24 & 68.71 & 67.75 & 54.46 \\
         OPERA* (Ours) & 81.54 & 72.74 & 76.80 & 63.95 \\
        \bottomrule
        \end{tabular}}
    \vspace{-8mm}
\end{table}

\begin{figure*}[!t]
  \begin{center}
    \centerline{\includegraphics[width=0.9\textwidth]{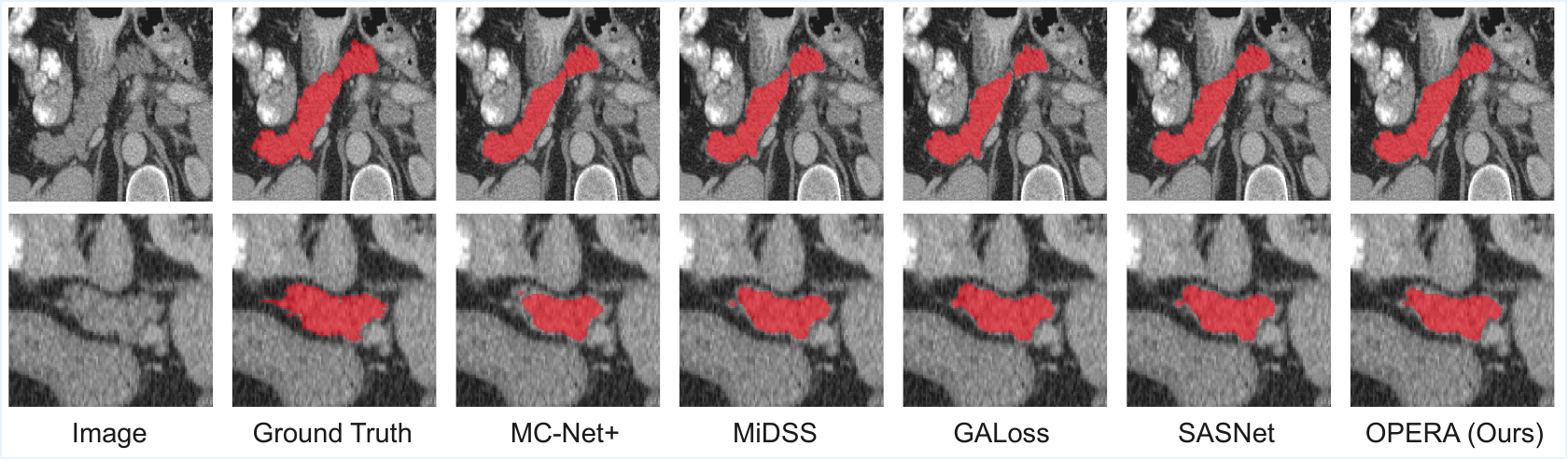}}
    \vspace{-4mm}
    \caption{Qualitative segmentation results on the Pancreas-CT dataset with 20\% labeled data.}
    \label{sup_results4}
  \end{center}
  \vspace{-7mm}
\end{figure*}
On the QaTa-COV19 COVID-Xray segmentation benchmark (Table~\ref{tab:performance_covid}), OPERA achieves the best performance with \textbf{82.12\%} Dice and \textbf{73.26\%} Jaccard, outperforming the strongest competing method STPNet. We attribute the gain to the ensemble of heterogeneous backbones, which reduces the typical over/under-segmentation artifacts of a single architecture and improves robustness to the low-contrast, fuzzy boundaries commonly observed in COVID-19 X-ray. As shown in Table~\ref{tab:performance_covid}, OPERA also ranks first, reaching \textbf{77.03\%} Dice and \textbf{64.39\%} Jaccard, surpassing STPNet (76.18\%/63.41\%) and other recent universal segmentation baselines. As visualized in Fig.~\ref{fig_results} (bottom row), OPERA better delineates scattered infection areas and preserves fine anatomical structures.

\vspace{-2mm}
\subsection{Universal Biomedical Image Segmentation with Limited Annotations}
Dense pixel-wise annotation is a major bottleneck for scaling universal segmentation to new organs and modalities. Following the limited-annotation protocol, we train all methods with only 20\% labeled data. We evaluate two representative 3D benchmarks: left atrial MRI (LA-MRI) and pancreas CT (Pancreas-CT) as shown in Table~\ref{tab:performance_semiseg}, Fig.~\ref{sup_results3}, and Fig.~\ref{sup_results4}. OPERA consistently outperforms recent semi-supervised baselines, suggesting that aggregating heterogeneous experts can provide more reliable predictions.
\begin{table}[!t]
\LARGE
    \centering
    \caption{The performance evaluation on left atrial MRI (LA-MRI) segmentation and pancreas CT (Pancreas-CT) segmentation with limited annotations (20\% labeled data). \textbf{Bold} indicates best performance and \underline{underline} shows second-best.}
    \vspace{-4mm}
    \label{tab:performance_semiseg}
    \resizebox{\linewidth}{!}{
        \begin{tabular}{l|cc|cc}
        \toprule
        \multirow{2}{*}{Method} & \multicolumn{2}{c|}{\textbf{LA-MRI}} & \multicolumn{2}{c}{\textbf{Pancreas-CT}} \\
                & Dice (\%) & Jaccard (\%) & Dice (\%) & Jaccard (\%)\\\hline
        MC-Net+ \textcolor{gray}{(MedIA-22)} & 90.12 & 82.12 & 79.05 & 65.83  \\
        PLGCL \textcolor{gray}{(ICCV-23)} & 90.01 & 82.04 & 78.41 & 65.17 \\
        AUSS \textcolor{gray}{(MedIA-24)} & 90.79 & 82.91 & 80.02 & 66.73 \\
        MiDSS \textcolor{gray}{(CVPR-24)} & 90.47 & 82.55 & 79.74 & 66.56 \\
        GALoss \textcolor{gray}{(ECCV-24)} & 90.39 & 82.46 & 80.21 & 66.92 \\
        TCA \textcolor{gray}{(CVPR-25)} & 90.67 & 83.16 & 80.73 & 67.88 \\
        SASNet \textcolor{gray}{(PR-26)}  & \underline{91.82} & \underline{84.93} & \underline{81.60} & \underline{69.39} \\\hline
        \textbf{OPERA (Ours)} & \textbf{92.71} & \textbf{85.62} & \textbf{82.44}  & \textbf{70.56} \\\rowcolor{gray!20}
        ~~~~~~$\mathbb{E}_1$: TransUNet & 86.75 & 76.94 & 73.61 & 60.89 \\\rowcolor{gray!20}
        ~~~~~~$\mathbb{E}_2$: U-Net & 85.61 & 75.53 & 70.83 & 57.12 \\\rowcolor{gray!20}
         ~~~~~~$\mathbb{E}_3$: V-Net & 86.03 & 76.06 & 71.52 & 57.68 \\
        \bottomrule
        \end{tabular}}
    \vspace{-6mm}
\end{table}

On LA-MRI, OPERA achieves \textbf{92.71\%} Dice and \textbf{85.62\%} Jaccard (Table~\ref{tab:performance_semiseg}), improving over the strongest baseline SASNet under the same 20\% labeling data. Fig.~\ref{sup_results3} illustrates the cases with an irregularly shaped atrium where MC-Net+ produces fragmented predictions and MiDSS misses a portion of the target region. OPERA consistently preserves the thin atrial boundaries and maintains geometric consistency, producing more complete and anatomically plausible segmentations despite the limited supervision. OPERA also ranks first with \textbf{82.44\%} Dice and \textbf{70.56\%} Jaccard on Pancreas-CT. It is particularly challenging due to low contrast and highly variable pancreas shapes. Fig.~\ref{sup_results4} shows a challenging case with elongated pancreatic anatomy where all other baseline methods produce incomplete segmentations that miss portions of the tail region. The bottom row depicts a case with low tissue contrast where OPERA can also segment the structures well. OPERA recovers more complete pancreas structures while suppressing spurious foreground regions, yielding cleaner boundaries and better coverage.

\begin{figure}[!t]
  \begin{center}
    \centerline{\includegraphics[width=0.9\columnwidth]{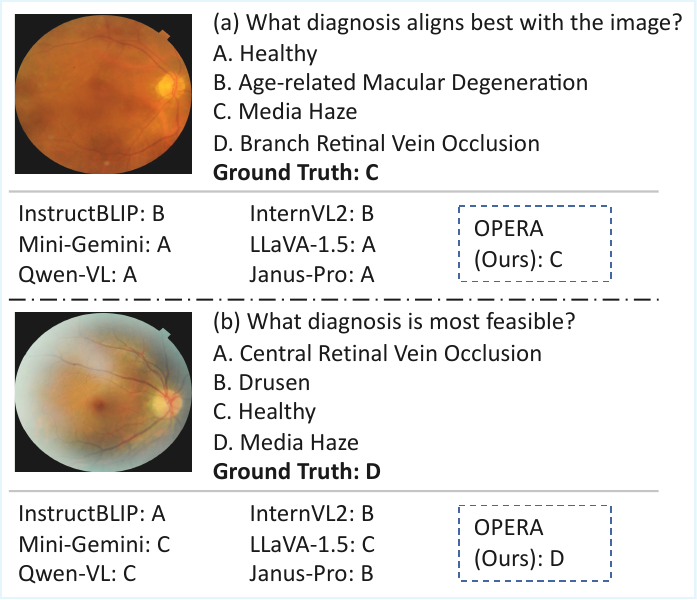}}
    \vspace{-5mm}
    \caption{Representative case studies of multimodal LLM diagnosis on retinal fundus images. Each case presents a multiple-choice question, the ground-truth disease label, and the predictions from several representative multimodal LLMs. OPERA corrects common confusions (e.g., AMD vs. media haze) by aggregating complementary expert opinions and can produce correct diagnosis from failure predictions.}
    \label{fig_results3}
  \end{center}
  \vspace{-10mm}
\end{figure}

\begin{table*}[!ht]
\LARGE
    \vspace{-2mm}
    \centering
    \caption{Comparison of multiple-choice accuracy in multimodal large language models on ten representative diseases. OPERA utilizes three expert agents with lower performance including InstructBLIP, Mini-Gemini, and Qwen-VL. \textbf{Bold} indicates best performance and \underline{underline} shows second-best.} 
    \vspace{-4mm}
    \resizebox{\textwidth}{!}{%
    \begin{tabular}{l|c|c|c|c|c|c|c|c|c|c|c}
    \toprule[0.7pt]
    Method & AMD & Cataract & CSR & DR & Glaucoma & Media Haze& Myopia & Retinitis & DME & Tessellation & \textbf{Average $\uparrow$}  \\
    \midrule
    InstructBLIP~\cite{instructblip} & 80.17\% & 80.00\% & 0.00\% & 76.51\% & 59.30\% & 16.13\% & 44.25\% & 11.11\% & 63.79\% & 41.67\% & 47.29\% \\
    Mini-Gemini~\cite{li2024mini} & 76.61\% & 85.00\% & 14.29\% & 79.87\% & 67.90\% & 38.71\% & 58.41\% & 33.33\% & 60.34\% & 33.33\% & 54.78\% \\
    Qwen-VL~\cite{bai2023qwen} & 81.87\% & 75.00\% & 28.57\% & 80.54\% & 78.40\% & 54.84\% & 76.55\% & 22.22\% & 84.48\% & 25.00\% & 60.75\% \\
    InternVL2~\cite{chen2024internvl} & 81.29\% & 85.00\% & \underline{71.43\%} & \textbf{94.63\%} & 89.51\% & \underline{64.52\%} & 88.05\% & \underline{44.44\%} & 87.93\% & \underline{66.67\%} & \underline{77.35\%} \\
    LLaVA-1.5~\cite{liu2024visual} & 83.04\% & \underline{90.00\%} & 42.86\% & 87.25\% & \underline{91.36\%} & 48.39\% & \underline{88.50\%} & 44.44\% & \textbf{93.10\%} & 58.33\% & 72.73\% \\
    Janus-Pro~\cite{chen2025janus} & \textbf{88.30\%} & 75.00\% & 42.86\% & 93.29\% & 90.74\% & 58.06\% & 87.17\% & 33.33\% & 62.07\% & 58.33\% & 68.92\% \\\hline
    \textbf{OPERA (Ours)} & \underline{87.13\%} & \textbf{90.00\%} & \textbf{71.43\%} & \underline{93.96\%} & \textbf{91.36\%} & \textbf{67.74\%} & \textbf{88.94\%} & \textbf{44.44\%} & \underline{91.38\%} & \textbf{66.67\%} & \textbf{79.31\%}\\
    \bottomrule
    \end{tabular}
    }
  \label{Tab:MLLM}%
  \vspace{-4mm}
\end{table*}

\vspace{-1mm}
\subsection{Universal Multimodal LLMs Diagnosis}
We further study whether OPERA can serve as a general multimodal assistant for clinical diagnosis, where the model must choose the most feasible disease label given a biomedical image and a set of candidate options. We evaluate representative multimodal LLMs on a ten-disease ophthalmic benchmark with the same multiple-choice prompting protocol. To apply OPERA's probabilistic fusion framework to MLLMs, we extract token-level log-probabilities from each model's output. Specifically, for a multiple-choice question with 4 candidate answers $\{A, B, C, D\}$, we prompt the MLLM to output the single answer token and retrieve log-probabilities $\ell_c$ assigned to each option token from the model's final softmax layer. These log-probabilities are then converted to normalized probabilities for use. Table~\ref{Tab:MLLM} reports multiple-choice accuracy for each disease and the overall average. Despite only using three relatively weaker models (InstructBLIP, Mini-Gemini, and Qwen-VL) as experts, OPERA achieves the best average accuracy of \textbf{79.31\%}, outperforming stronger individual MLLMs such as InternVL and Janus-Pro. OPERA is consistently competitive across categories and attains the best performance on most diseases. Fig.~\ref{fig_results3} provides representative examples, showing that OPERA can have correct diagnosis even when single model is misled by low-contrast or hazy appearances.

\subsection{Ablation Studies}
\paragraph{Ablation Study of Proposed Strategies.}
We first ablate each component in our inference pipeline on OIA-DDR classification (Table~\ref{tab:ablation_strategy}) and on COVID-19 segmentation (Table~\ref{tab:ablation_strategy_seg}). Starting from the EPM initialization, introducing DAA yields a clear improvement compared to the simple probability averaging. Adding PECC further boosts performance and achieves the best performance with IER. The same combination also improves Dice from 80.52\% to \textbf{82.12}\% on QaTa-COV19 and from 75.94\% to \textbf{77.03}\% on MosMedData+.

\begin{table}[!t]
\LARGE
    \centering
    \caption{Ablation study of different proposed strategies combination on OIA-DDR. The expert agent CLIP-\textit{X} denotes using \textit{X} as the vision encoder with CLIP framework.} %3759
    \vspace{-4mm}
    \label{tab:ablation_strategy}
    \resizebox{\linewidth}{!}{
        \begin{tabular}{l|ccc|ccc}
        \toprule
        \multirow{2}{*}{Method} & \multicolumn{3}{c|}{\textbf{Linear Probe (\%)}} & \multicolumn{3}{c}{\textbf{Fully Fine-tune (\%)}} \\
                & AUC   & ACC   & mAP   & AUC   & ACC   & mAP \\ 
        \hline\rowcolor{gray!20}
        $\mathbb{E}_1$: CLIP-ViT& 73.30 & 55.41 & 39.21 & 85.29 & 71.75 & 49.91 \\\rowcolor{gray!20}
        $\mathbb{E}_2$: CLIP-ResNet50 & 71.69 & 54.91 & 37.52 & 80.53 & 69.78 & 46.22 \\\rowcolor{gray!20}
        $\mathbb{E}_3$: CLIP-DenseNet121 & 72.73 & 55.20 & 38.19 & 81.46 & 69.94 & 47.09 \\\rowcolor{gray!20}
        Probability Averaging Ensemble& 74.26 & 57.14 & 40.33 & 85.33 & 71.85 & 50.14 \\\hline
        EPM & 81.69 & 64.11 & 48.02 & 85.75 & 72.17 & 50.81\\
        EPM + DAA & 82.83 & 66.80 & 49.15 & 86.03 & 72.47 & 53.64 \\
        EPM + DAA + PECC & 84.01 & 67.86 & 49.54 & 87.24 & 72.81 & 55.08\\
        DAA + PECC + IER & 83.18 & 67.07 & 49.33 & 86.30 & 72.52 & 54.29 \\
        EPM + DAA + PECC + IER & \textbf{84.95} & \textbf{69.57} & \textbf{49.81} & \textbf{88.06} & \textbf{72.89} & \textbf{55.90}\\
        \bottomrule
        \end{tabular}}
    \vspace{-4mm}
\end{table}

\begin{table}[!t]
\LARGE
    \centering
    \caption{Ablation study of different proposed strategies combination on COVID-19 X-ray (QaTa-COV19) and COVID-19 CT (MosMedData+) Segmentation.}
    \vspace{-4mm}
    \label{tab:ablation_strategy_seg}
    \resizebox{\linewidth}{!}{
        \begin{tabular}{l|cc|cc}
        \toprule
        \multirow{2}{*}{Method} & \multicolumn{2}{c|}{\textbf{QaTa-COV19}} & \multicolumn{2}{c}{\textbf{MosMedData+}} \\
                & Dice (\%) & Jaccard (\%) & Dice (\%) & Jaccard (\%) \\ 
        \hline\rowcolor{gray!20}
        $\mathbb{E}_1$: TransUNet & 78.63 & 69.13 & 71.24 & 58.44 \\\rowcolor{gray!20}
        $\mathbb{E}_2$: U-Net & 79.02 & 69.46 & 64.60 & 50.73 \\\rowcolor{gray!20}
        $\mathbb{E}_3$: V-Net & 78.24 & 68.71 & 67.75 & 54.46 \\\rowcolor{gray!20}
        Probability Averaging Ensemble & 79.37 & 69.81 & 71.95 & 59.13 \\\hline
        EPM & 80.52 & 71.35 & 75.94 & 63.20\\
        EPM + DAA & 81.47 & 72.46 & 76.51 & 63.68\\
        EPM + DAA + PECC & 81.85 & 73.10 & 76.87 & 64.06\\
        DAA + PECC + IER & 81.54 & 72.74 & 76.80 & 63.95 \\
        EPM + DAA + PECC + IER & \textbf{82.12} & \textbf{73.26} & \textbf{77.03}  & \textbf{64.39}\\
        \bottomrule
        \end{tabular}}
    \vspace{-4mm}
\end{table}

\vspace{-1mm}
\paragraph{Ablation Study of expert agent Composition.} We also study how the choice and number of experts affect OPERA (Table~\ref{tab:ablation_composition_cls}). Each individual CLIP expert performs notably worse than the ensemble. Composing two experts already brings large gains such as $\mathbb{E}_1+\mathbb{E}_3$. However, using all three experts consistently yields the best results, indicating that the experts provide complementary representations.

\begin{table}[!t]
\large
    \centering
    \caption{Ablation study of different expert agent composition on OIA-DDR. The expert agent CLIP-\textit{X} denotes using \textit{X} as the vision encoder with CLIP framework.} %3759
    \vspace{-4mm}
    \label{tab:ablation_composition_cls}
    \resizebox{\linewidth}{!}{
        \begin{tabular}{l|ccc|ccc}
        \toprule
        \multirow{2}{*}{Method} & \multicolumn{3}{c|}{\textbf{Linear Probe (\%)}} & \multicolumn{3}{c}{\textbf{Fully Fine-tune (\%)}} \\
                & AUC   & ACC   & mAP   & AUC   & ACC   & mAP \\ 
        \hline\rowcolor{gray!20}
        $\mathbb{E}_1$: CLIP-ViT& 73.30 & 55.41 & 39.21 & 85.29 & 71.75 & 49.91 \\\rowcolor{gray!20}
        $\mathbb{E}_2$: CLIP-ResNet50 & 71.69 & 54.91 & 37.52 & 80.53 & 69.78 & 46.22 \\\rowcolor{gray!20}
        $\mathbb{E}_3$: CLIP-DenseNet121 & 72.73 & 55.20 & 38.19 & 81.46 & 69.94 & 47.09 \\\hline
        ~~~~~~~~$\mathbb{E}_1$ + $\mathbb{E}_2$& 80.32 & 61.03 & 46.68 & 86.03 & 71.88 & 50.76\\
        ~~~~~~~~$\mathbb{E}_1$ + $\mathbb{E}_3$& 81.05 & 62.60 & 47.59 & 86.73 & 72.36 & 51.41\\
        ~~~~~~~~$\mathbb{E}_2$ + $\mathbb{E}_3$& 79.20 & 60.49 & 46.33 & 83.56 & 68.72 & 47.94\\
        ~~~~$\mathbb{E}_1$ + $\mathbb{E}_2$ + $\mathbb{E}_3$ & \textbf{84.95} & \textbf{69.57} & \textbf{49.81} & \textbf{88.06} & \textbf{72.89} & \textbf{55.90}\\
        \bottomrule
        \end{tabular}}
    \vspace{-4mm}
\end{table}

\vspace{-1mm}
\paragraph{Ablation Study of Blend Factors.} Finally, we ablate the blend factors $(\alpha,\beta)$ used in DAA for segmentation (Table~\ref{tab:ablation_factors}). Moderate blending performs best: setting $\alpha=0.10$ and $\beta=0.10$ achieves the highest performance on both datasets. In contrast, skewed blends (0.10, 0.05) or (0.05, 0.10) lead to slightly lower accuracy. And overly aggressive blending degrades performance, which indicates DAA benefits from balanced but conservative parameter updates.

\begin{table}[!t]
\large
    \centering
    \caption{Ablation study of blend factors in Distribution-Aware Adaptation on the COVID-19 X-ray (QaTa-COV19) and COVID-19 CT (MosMedData+) segmentation datasets.}
    \vspace{-4mm}
    \label{tab:ablation_factors}
    \resizebox{\linewidth}{!}{
        \begin{tabular}{l|cc|cc}
        \toprule
        \multirow{2}{*}{Method} & \multicolumn{2}{c|}{\textbf{QaTa-COV19}} & \multicolumn{2}{c}{\textbf{MosMedData+}} \\
                & Dice (\%) & Jaccard (\%) & Dice (\%) & Jaccard (\%) \\ 
        \hline\rowcolor{gray!20}
        $\mathbb{E}_1$: TransUNet & 78.63 & 69.13 & 71.24 & 58.44 \\\rowcolor{gray!20}
        $\mathbb{E}_2$: U-Net & 79.02 & 69.46 & 64.60 & 50.73 \\\rowcolor{gray!20}
        $\mathbb{E}_3$: V-Net & 78.24 & 68.71 & 67.75 & 54.46 \\\hline
        $\alpha=0.10$, $\beta=0.05$ & 82.10 & 73.25 & 76.91 & 64.32\\
        $\alpha=0.05$, $\beta=0.10$ & 82.05 & 73.12 & 76.97 & 64.35\\
        $\alpha=0.10$, $\beta=0.10$ & \textbf{82.12} & \textbf{73.26} & \textbf{77.03}  & \textbf{64.39}\\
        $\alpha=0.20$, $\beta=0.20$ & 81.97 & 73.01 & 76.72 & 64.11 \\
        \bottomrule
        \end{tabular}}
    \vspace{-6mm}
\end{table}

\vspace{-2mm}
\section{Conclusion}
In this work, we introduce \textbf{OPERA} (Offline Policy-guided Expert Routing and Adaptation), an offline-calibrated zero-retraining coordination framework that composes heterogeneous pre-trained agents at the inference time. By coupling validation-driven weight initialization with confidence calibration and test-time adaptive fusion, OPERA turns complementary expert behaviors into a single deployable predictor without any training parameter updates. Across nine benchmarks spanning fundus image, chest X-ray, CT, MRI, and multimodal diagnostic evaluation, OPERA consistently outperformed strong single-model and ensemble baselines in both classification and segmentation. Promising next steps include cost-aware expert selection and routing, extending the experts pool to broader modalities and tasks, and subgroup analyses for better real-world deployment and clinical interpretability.\\
\noindent\textbf{Acknowledgments}: This work was supported in part by the National Natural Science Foundation of China under Grant 62471418 and Fujian Provincial Natural Science Foundation of China under Grant 2024J01058.

\setlength{\abovedisplayskip}{3pt}
\setlength{\belowdisplayskip}{3pt}
\setlength{\abovedisplayshortskip}{3pt}
\setlength{\belowdisplayskip}{3pt}

\bibliographystyle{ACM-Reference-Format}
\bibliography{main}

\clearpage
\newpage

\appendix
\section{Theoretical Analysis of OPERA}
In this section, we provide theoretical justifications for the design choices in OPERA. We analyze how the multi-level weighting strategy, confidence calibration, and test-time adaptation collectively contribute to improved ensemble performance. Our analysis draws on classical bias-variance decomposition, expert diversity theory, and calibration-aware fusion principles. OPERA also defines the fixed routing policy learned from a static validation set.

\subsection{Bias-Variance Decomposition for Weighted Ensembles}
We begin by analyzing how weighted averaging of expert predictions can reduce prediction error. We adopt a probabilistic perspective where randomness arises from the sampling of test instances $(x, y) \sim \mathcal{D}$ from the data distribution $\mathcal{D}$, rather than from the models themselves. Since OPERA uses fixed pre-trained expert agents, each $f^{(m)}(x)$ is a deterministic function of $x$; the variance and covariance terms below therefore capture variability across the input distribution rather than model stochasticity. Consider a regression setting where each expert agent $m$ produces predictions $f^{(m)}(x)$ for input $x$, and let $y$ denote the true label. The expected squared error of a weighted ensemble with weights $\mathbf{w} = (w_1, \ldots, w_M)^\top$ satisfying $\sum_{m=1}^{M} w_m = 1$ can be decomposed as follows.

\begin{theorem}[Weighted Ensemble Error Decomposition]
\label{thm:bias_variance}
Let $\bar{f}(x) = \sum_{m=1}^{M} w_m f^{(m)}(x)$ be the weighted ensemble prediction. Under expectations taken over $(x, y) \sim \mathcal{D}$, the expected squared error admits the decomposition:
\begin{equation}
\mathbb{E}_{(x,y) \sim \mathcal{D}}[(y - \bar{f}(x))^2] = \underbrace{\text{Bias}^2}_{\text{(i)}} + \underbrace{\text{Var}_{\text{ens}}}_{\text{(ii)}} + \underbrace{\sigma^2_y}_{\text{(iii)}},
\end{equation}
where:
\begin{align}
&\text{(i)} \quad \text{Bias}^2 
= \left(
\mathbb{E}_{x}[y|x] 
- \sum_{m=1}^{M} w_m \mathbb{E}_{x}[f^{(m)}(x)]
\right)^2, \\
&\text{(ii)} \quad \text{Var}_{\text{ens}} 
= \sum_{m=1}^{M} w_m^2 \text{Var}_{x}(f^{(m)}(x)) \\
&\quad + 2\sum_{m<m'} w_m w_{m'} 
\text{Cov}_{x}(f^{(m)}(x), f^{(m')}(x)), \\
&\text{(iii)} \quad \sigma^2_y 
= \mathbb{E}_{(x,y)}[(y - \mathbb{E}[y|x])^2] 
\quad \text{(irreducible noise)}.
\end{align}
\end{theorem}

\begin{proof}
We provide a complete derivation. Let $\mu(x) = \mathbb{E}[y|x]$ denote the true conditional mean. The squared error decomposes as:
\begin{align}
\mathbb{E}[(y - \bar{f}(x))^2] 
&= \mathbb{E}[(y - \mu(x) + \mu(x) - \bar{f}(x))^2] \\
&= \mathbb{E}[(y - \mu(x))^2] 
+ \mathbb{E}[(\mu(x) - \bar{f}(x))^2] \\
&\quad + 2\mathbb{E}[(y - \mu(x))(\mu(x) - \bar{f}(x))].
\end{align}
The cross-term vanishes since $\mathbb{E}[y - \mu(x) | x] = 0$, yielding:
\begin{equation}
\mathbb{E}[(y - \bar{f}(x))^2] = \underbrace{\mathbb{E}[(y - \mu(x))^2]}_{\sigma^2_y} + \underbrace{\mathbb{E}[(\mu(x) - \bar{f}(x))^2]}_{\text{reducible error}}.
\end{equation}
The reducible error term can be further expanded. Define $\bar{\mu} = \mathbb{E}_x[\mu(x)]$ and $\bar{f}_m = \mathbb{E}_x[f^{(m)}(x)]$. Then:
\begin{align}
\mathbb{E}[(\mu(x) - \bar{f}(x))^2] 
&= \mathbb{E}\left[
\left(\mu(x) - \sum_{m=1}^{M} w_m f^{(m)}(x)\right)^2
\right] \\
&= \mathbb{E}[\mu(x)^2] 
- 2\sum_{m=1}^{M} w_m \mathbb{E}[\mu(x) f^{(m)}(x)] \\
&\quad + \mathbb{E}\left[
\left(\sum_{m=1}^{M} w_m f^{(m)}(x)\right)^2
\right].
\end{align}
For the last term:
\begin{equation}
\begin{aligned}
\mathbb{E}\Bigg[
\left(\sum_{m=1}^{M} w_m f^{(m)}(x)\right)^2
\Bigg]
&= \sum_{m=1}^{M} w_m^2 \mathbb{E}[(f^{(m)}(x))^2] \\
&\quad + 2\sum_{m<m'} w_m w_{m'} 
\mathbb{E}[f^{(m)}(x) f^{(m')}(x)] \\
&= \sum_{m=1}^{M} w_m^2 
\left[\mathrm{Var}_x(f^{(m)}(x)) + \bar{f}_m^2\right] \\
&\quad + 2\sum_{m<m'} w_m w_{m'} 
\Big[
\mathrm{Cov}_x(f^{(m)}, f^{(m')}) \\
&\qquad\qquad + \bar{f}_m \bar{f}_{m'}
\Big].
\end{aligned}
\end{equation}

Collecting the mean terms yields the squared bias $(\bar{\mu} - \sum_m w_m \bar{f}_m)^2$, and the remaining terms give the variance-covariance structure, completing the decomposition.
\end{proof}

\paragraph{Implications for OPERA.} This decomposition reveals two key insights:
\begin{enumerate}
    \item \textbf{Variance reduction through diversification}: The ensemble variance term $\sum_{m=1}^{M} w_m^2 \text{Var}(f^{(m)})$ is minimized when weights are uniform, assuming equal individual variances. More importantly, when expert agents produce negatively correlated or uncorrelated errors (i.e., $\text{Cov}(f^{(m)}, f^{(m')}) \leq 0$), the covariance term further reduces ensemble variance. OPERA's use of architecturally diverse experts (ViT, ResNet, DenseNet) promotes such error decorrelation.
    
    \item \textbf{Bias-variance trade-off in weight selection}: Concentrating weights on better-performing models may reduce bias but increases the effective variance (since $\sum_m w_m^2$ increases with weight concentration). OPERA's EPM navigates this trade-off using temperature-scaled softmax with $\tau = 10.0$, which moderately concentrates weights on higher-performing models while preserving sufficient diversity to benefit from variance reduction.
\end{enumerate}

\subsection{Expert Diversity and Error Decorrelation}

The effectiveness of ensemble methods critically depends on the diversity of constituent models~\cite{liu2019accurate}. We formalize this through the following pairwise analysis of prediction disagreement.

\begin{definition}[Pairwise Disagreement]
For two models $m$ and $m'$, the pairwise disagreement on sample $i$ is defined as:
\begin{equation}
d_{m,m'}(i) = \frac{1}{C} \sum_{c=1}^{C} \left(p^{(m)}_{i,c} - p^{(m')}_{i,c}\right)^2.
\end{equation}
The ensemble diversity is the expected pairwise disagreement:
\begin{equation}
\mathcal{D} = \frac{2}{M(M-1)} \sum_{m<m'} \mathbb{E}_i[d_{m,m'}(i)].
\end{equation}
where $M=2$ denotes the number of models in each pair (distinct from the $M=3$ ensemble expert agents defined earlier).
\end{definition}

\begin{proposition}[Diversity-Error Relationship]
\label{prop:diversity}
Let $\bar{e}$ denote the average individual model error and $e_{\text{ens}}$ denote the ensemble error. Under squared loss, the ensemble error satisfies:
\begin{equation}
e_{\text{ens}} = \bar{e} - \mathcal{D},
\end{equation}
where the ensemble uses uniform weights.
\end{proposition}

As known as the ambiguity decomposition \cite{Kuncheva2003}, this result implies that higher diversity directly translates to lower ensemble error, provided individual model accuracy is maintained. OPERA leverages this principle by combining architecturally diverse experts: ViT (attention-based), ResNet50 (residual connections), and DenseNet121 (dense connections). These architectural differences induce different inductive biases, leading to decorrelated errors across models.

For the Instance-level Expert Routing (IER) module, we exploit the relationship between inter-model agreement and prediction reliability. The agreement score defined in Eq.~(\ref{eq14}) of the main text quantifies instantaneous diversity:
\begin{equation}
a_i = \frac{2}{M(M-1)}\sum_{m=1}^{M-1}\sum_{m'=m+1}^{M} \text{sim}(m, m', i).
\end{equation}

\paragraph{Agreement-Reliability Heuristic.}
We provide intuition for why high inter-model agreement suggests reliable ensemble predictions, based on a standard result from ensemble theory.

\begin{proposition}[Majority Vote Reliability]
\label{prop:agreement}
Consider an ensemble of $M$ models making independent predictions, where each model has probability $p > 0.5$ of being correct. The probability that the majority vote is correct is given by:
\begin{equation}
\Pr(\text{majority correct}) = \sum_{k=\lceil M/2 \rceil}^{M} \binom{M}{k} p^k (1-p)^{M-k}.
\end{equation}
For $M = 3$ models with individual accuracy $p$, this simplifies to:
\begin{equation}
\Pr(\text{majority correct}) = 3p^2(1-p) + p^3 = 3p^2 - 2p^3.
\end{equation}
\end{proposition}

This standard binomial result shows that when individual models are better than random ($p > 0.5$), the majority vote accuracy exceeds individual accuracy, and this advantage grows with $p$. While the independence assumption is often violated in practice (since models may share similar biases or training data), high empirical agreement among diverse expert architectures provides a practical signal that the prediction is likely reliable.

\paragraph{Practical Interpretation for IER.} Rather than relying on the idealized independence assumption, IER uses inter-model agreement $a_i$ as a \textit{heuristic indicator} of prediction reliability:
\begin{itemize}
    \item \textbf{High agreement} ($a_i > 0.90$ and $\sigma^2_i < 0.03$): When architecturally diverse experts converge on similar predictions, this empirical consensus suggests the instance lies in a region of input space where all models perform well. In such cases, IER uses the validation-derived base weights $\mathbf{w}_{\text{model}}$, trusting the ensemble's collective judgment.
    \item \textbf{Low agreement} ($a_i < 0.60$ or $\sigma^2_i > 0.12$): When experts disagree substantially, the instance likely presents ambiguity or lies near decision boundaries. Here, IER applies entropy-based weighting to prioritize more confident predictions, serving as a tie-breaking mechanism.
\end{itemize}
The specific thresholds (0.90, 0.60, 0.03, 0.12) were selected based on empirical validation across our benchmark datasets, and the effectiveness of this conditional weighting strategy is demonstrated in the ablation studies (Table~\ref{tab:ablation_strategy} and Table~\ref{tab:ablation_strategy_seg}).

\subsection{Calibration Theory and Optimal Fusion}

Proper calibration is essential for meaningful probability fusion~\cite{cao2024predictive}. We formalize the role of Per-Expert Confidence Calibration (PECC) in achieving well-calibrated ensemble predictions.

\begin{definition}[Calibration Error]
A model is perfectly calibrated if for all predicted probability values $p \in [0,1]$:
\begin{equation}
\Pr(Y = 1 | \hat{p} = p) = p.
\end{equation}
The Expected Calibration Error (ECE) measures deviation from perfect calibration:
\begin{equation}
\text{ECE} = \mathbb{E}_{\hat{p}}\left[|\Pr(Y = 1 | \hat{p}) - \hat{p}|\right].
\end{equation}
\end{definition}

Temperature scaling transforms predictions via $\hat{p}_T = \sigma(z/T)$ where $z = \sigma^{-1}(\hat{p})$ is the logit. The following result characterizes optimal temperature selection.

\begin{proposition}[Temperature Scaling Properties]
\label{prop:temperature}
Let $\hat{p}$ be a model's predicted probability with sharpness $s = \mathbb{E}[|\hat{p} - 0.5|]$. Temperature scaling with $T > 1$ (softening) reduces sharpness, while $T < 1$ (sharpening) increases sharpness. Specifically:
\begin{equation}
\frac{\partial s_T}{\partial T} < 0 \quad \text{for all } T > 0,
\end{equation}
where $s_T$ is the sharpness after temperature scaling.
\end{proposition}

\begin{proof}
The temperature-scaled probability is $\hat{p}_T = \sigma(z/T) = 1/(1 + e^{-z/T})$. Taking the derivative with respect to $T$:
\begin{equation}
\frac{\partial \hat{p}_T}{\partial T} = \frac{z \cdot e^{-z/T}}{T^2(1 + e^{-z/T})^2} = -\frac{z}{T^2} \hat{p}_T(1 - \hat{p}_T).
\end{equation}
For $z > 0$ (i.e., $\hat{p} > 0.5$), increasing $T$ decreases $\hat{p}_T$ toward 0.5. For $z < 0$, increasing $T$ increases $\hat{p}_T$ toward 0.5. Thus, the distance from 0.5 decreases with increasing $T$, implying $\partial s_T / \partial T < 0$.
\end{proof}

PECC's adaptive temperature selection addresses the observation that different models exhibit different calibration characteristics. Overconfident models ($s_m > 0.4$) receive $T_m = 1.5$ to soften predictions, while underconfident models ($s_m < 0.1$) receive $T_m = 0.7$ to sharpen predictions. This adaptive calibration ensures that the subsequent weighted fusion operates on commensurately scaled probabilities.

\begin{theorem}[Calibrated Fusion Optimality]
\label{thm:fusion}
Consider an ensemble of $M$ calibrated models with class-conditional accuracies $\{q_{m,c}\}_{m,c}$. The Bayes-optimal fusion weights that minimize expected misclassification rate satisfy~\cite{shi2003study}:
\begin{equation}
w^*_{m,c} \propto \log\left(\frac{q_{m,c}}{1 - q_{m,c}}\right),
\end{equation}
which is monotonically increasing in $q_{m,c}$.
\end{theorem}

This result justifies EPM's use of class-level AUC scores to determine weights: models with higher discriminative performance for specific classes should receive proportionally higher weights for those classes. The temperature-scaled softmax in Eq.~(\ref{eq4}) approximates this optimal weighting scheme.

\subsection{Test-Time Adaptation: Convergence and Consistency}

Distribution-Aware Adaptation (DAA) adapts class-level weights during inference using unlabeled test data. We analyze the convergence properties of this adaptation mechanism.

\begin{assumption}[Bounded Confidence Scores]
The confidence scores $\text{conf}_{m,c}^{(k)}$ defined in Eq.~(\ref{eq9}) are bounded: $\text{conf}_{m,c}^{(k)} \in [0, 0.5]$ for all models $m$, classes $c$, and batches $k$.
\end{assumption}

\begin{proposition}[DAA Convergence]
\label{prop:ttwt_convergence}
Under the bounded confidence assumption, the DAA update rule (Eq.~(\ref{eq11})) with $\alpha, \beta \in (0, 1)$ and $\alpha + \beta < 1$ converges to a stationary weight distribution. Specifically, as $K \to \infty$:
\begin{equation}
\mathbf{W}'_{\text{class}} \to \frac{(1-\alpha-\beta)}{1-\alpha} \mathbf{W}^{\text{orig}}_{\text{class}} + \frac{\beta}{1-\alpha} \mathbf{W}^{\text{conf},\infty}_{\text{class}},
\end{equation}
where $\mathbf{W}^{\text{conf},\infty}_{\text{class}}$ is the confidence-based weight computed from the limiting confidence statistics.
\end{proposition}

\begin{proof}
The update rule can be written in matrix form as:
\begin{equation}
\mathbf{W}^{(k+1)}_{\text{class}} = \alpha \mathbf{W}^{(k)}_{\text{class}} + \beta \mathbf{W}^{\text{conf},(k)}_{\text{class}} + (1-\alpha-\beta) \mathbf{W}^{\text{orig}}_{\text{class}}.
\end{equation}
This is a contraction mapping with coefficient $\alpha < 1$. By the Banach fixed-point theorem~\cite{gordji2017orthogonal}, it converges to a unique fixed point. At the fixed point $\mathbf{W}^*$:
\begin{equation}
\mathbf{W}^* = \alpha \mathbf{W}^* + \beta \mathbf{W}^{\text{conf},\infty}_{\text{class}} + (1-\alpha-\beta) \mathbf{W}^{\text{orig}}_{\text{class}}.
\end{equation}
Solving for $\mathbf{W}^*$ yields the stated result.
\end{proof}

The convergence result shows that DAA maintains a weighted combination of validation-based knowledge and test-time observations. With $\alpha = \beta = 0.1$, the limiting weights assign approximately major influence to validation-based weights and minor to test-time confidence statistics, ensuring stability while allowing domain adaptation.

\begin{proposition}[Domain Shift Adaptation]
\label{prop:domain_shift}
Let $P_{\text{train}}$ and $P_{\text{test}}$ denote the training and test distributions, respectively. Under covariate shift where $P_{\text{train}}(Y|X) = P_{\text{test}}(Y|X)$ but $P_{\text{train}}(X) \neq P_{\text{test}}(X)$, DAA's confidence-based adaptation provides consistent weight estimates if the confidence scores are monotonically related to model accuracy.
\end{proposition}

This proposition justifies why confidence-based weighting can improve performance under distribution shift: models that maintain higher confidence on the test distribution are likely those whose learned features remain discriminative under shift.

\subsection{Instance-Level Weighting: Uncertainty Quantification}
Instance-level Expert Routing (IER) produces sample-specific weights based on inter-model agreement and per-model uncertainty. Rather than claiming strict theoretical optimality, we provide a decision-theoretic motivation for this approach while acknowledging the heuristic nature of the practical implementation.

\begin{definition}[Predictive Uncertainty]
For model $m$ and sample $i$, the predictive uncertainty is quantified by the averaged binary entropy:
\begin{equation}
\begin{aligned}
h_{i,m} 
&= \frac{1}{C}\sum_{c=1}^{C} H(p'^{(m)}_{i,c}) \\
&= -\frac{1}{C}\sum_{c=1}^{C} 
\Big[
p'^{(m)}_{i,c} \log p'^{(m)}_{i,c} \\
&\quad + (1-p'^{(m)}_{i,c}) \log(1-p'^{(m)}_{i,c})
\Big].
\end{aligned}
\end{equation}
\end{definition}

\paragraph{Theoretical Motivation.} We begin by establishing a connection between predictive entropy and expected loss under idealized conditions.

\begin{proposition}[Entropy-Loss Decomposition]
\label{prop:entropy_loss}
Consider a model $m$ with predictive distribution $p'^{(m)}(y|x)$. Under logarithmic loss, the expected loss decomposes as:
\begin{equation}
\mathcal{L}_m = -\mathbb{E}_{y \sim p_{\text{true}}}[\log p'^{(m)}(y|x)] = H(p_{\text{true}}) + D_{\text{KL}}(p_{\text{true}} \| p'^{(m)}),
\end{equation}
where $H(p_{\text{true}})$ is the entropy of the true distribution and $D_{\text{KL}}$ denotes the Kullback-Leibler divergence~\cite{ando2011predictive}.
\end{proposition}

This decomposition shows that a model's expected loss equals the irreducible true entropy plus its divergence from the true distribution. Under the idealized assumption of perfect calibration ($p'^{(m)} = p_{\text{true}}$), one could derive optimal weights $w^*_m \propto \exp(-\lambda \cdot \mathcal{L}_m)$ that minimize ensemble loss. However, we emphasize two critical limitations: (1) perfect calibration is rarely achieved in practice, and (2) calibration does not imply that lower-entropy predictions are more accurate, which means a model may be confidently wrong on specific instances. Therefore, we do not claim that entropy-based weighting is theoretically optimal; rather, it serves as a practical heuristic motivated by the intuition that confident predictions from diverse experts may carry more informative signals when models disagree.

\paragraph{Heuristic Implementation.} Given these limitations, IER adopts an empirically-motivated inverse-entropy weighting scheme in high-disagreement scenarios:
\begin{equation}
w^{(i)}_m \propto \exp\left(\frac{\gamma}{h_{i,m} + \epsilon}\right).
\end{equation}
This functional form differs from the theoretically-suggested $\exp(-\lambda \cdot H)$. The inverse-entropy formulation $\exp(\gamma / h_{i,m})$ was selected based on empirical performance, as it provides stronger differentiation between high-confidence and low-confidence predictions compared to the negative-entropy form, particularly when entropy values are small. The parameter $\gamma = 2.0$ controls the sensitivity to entropy differences, and $\epsilon$ ensures numerical stability.

Importantly, this heuristic is applied \textit{conditionally}: only when inter-model agreement is low ($a_i < 0.60$) or prediction variance is high ($\sigma^2_i > 0.12$). In high-agreement cases, IER defaults to the validation-derived weights $\mathbf{w}_{\text{model}}$, which are grounded in empirical AUC performance rather than entropy-based reasoning. This conditional design mitigates the risk of overconfident but incorrect predictions dominating the ensemble, and the effectiveness of this heuristic is validated empirically through ablation studies presented in Table~\ref{tab:ablation_strategy} and Table~\ref{tab:ablation_strategy_seg}.

\subsection{Hierarchical Weight Combination}

The final prediction in OPERA combines weights at three levels: model-level ($\mathbf{w}_{\text{model}}$), class-level ($\mathbf{W}_{\text{class}}$), and instance-level ($\mathbf{w}^{(i)}$). We justify this hierarchical structure with the following proposition and corollary.

\begin{proposition}[Hierarchical Decomposition]
\label{prop:hierarchical}
The optimal fusion weight for model $m$, class $c$, and instance $i$ can be decomposed as:
\begin{equation}
W^*_{i,m,c} = \underbrace{\pi_m}_{\text{model prior}} \cdot \underbrace{\rho_{m,c}}_{\text{class expertise}} \cdot \underbrace{\eta_{i,m}}_{\text{instance reliability}},
\end{equation}
where $\pi_m$ captures overall model quality, $\rho_{m,c}$ captures model-class affinity, and $\eta_{i,m}$ captures instance-specific model suitability.
\end{proposition}

OPERA's final weight computation (Eq.~(\ref{eq18})) directly implements this decomposition:
\begin{equation}
W_{\text{final},i,m,c} = w_{\text{model},m} \cdot {w}'_{\text{class,m,c}} \cdot w^{(i)}_m.
\end{equation}

This multiplicative combination ensures that all three factors contribute to the final weight: a model must perform well overall, demonstrate expertise for the specific class, and produce reliable predictions for the particular instance to receive high weight in the ensemble.

\begin{corollary}[Robustness to Expert Failures]
\label{cor:robustness}
The hierarchical weighting provides robustness to expert failures. If model $m$ fails on instance $i$ (producing high-entropy predictions), then $w^{(i)}_m \approx 0$ regardless of its model-level and class-level weights, preventing the failed prediction from corrupting the ensemble output.
\end{corollary}

This robustness property is particularly valuable in medical imaging, where individual models may fail on out-of-distribution samples or challenging cases. The instance-level weighting acts as a safety mechanism that automatically downweights unreliable predictions.

\subsection{Computational Complexity Analysis}

We conclude with an analysis of OPERA's computational overhead.

\begin{proposition}[Inference Complexity]
\label{prop:complexity}
For an ensemble of $M$ models with $C$ classes and batch size $B$, OPERA's inference-time operations have the following complexity:
\begin{itemize}
    \item PECC temperature scaling: $\mathcal{O}(MBC)$
    \item DAA confidence computation: $\mathcal{O}(MBC)$
    \item IER agreement computation: $\mathcal{O}(M^2BC)$
    \item Weight combination and prediction: $\mathcal{O}(MBC)$
\end{itemize}
The total additional overhead is $\mathcal{O}(M^2BC)$, which is negligible compared to the forward pass cost of expert agents.
\end{proposition}

Since OPERA uses $M = 3$ experts and the quadratic term in $M$ is small, the computational overhead of the ensemble fusion is minimal. The primary computational cost remains in running the expert agents' forward passes, which can be parallelized across GPUs. This efficiency makes OPERA practical for deployment in clinical settings where inference latency is a concern.

\paragraph{Empirical Efficiency Considerations.} We note that the inference efficiency of OPERA is not a practical bottleneck for clinical deployment due to several factors. First, the fusion operations (PECC, DAA, IER, and weight combination) contribute negligible latency compared to the expert forward passes, accounting for less than 1\% of the total inference time. Second, while running multiple expert agents sequentially incurs a proportional increase in latency, modern GPU architectures readily support parallel execution of multiple models using CUDA streams or multi-GPU setups, reducing the effective latency overhead to a small fraction above that of the slowest single expert. Third, GPU memory consumption for the full ensemble remains well within the capacity of standard clinical-grade hardware. Finally, for deployment scenarios with strict real-time constraints, practitioners can adopt a cascaded inference strategy where a lightweight expert first screens inputs and the full ensemble is invoked only for uncertain cases (e.g., when prediction entropy exceeds a threshold), thereby reducing average latency while preserving accuracy on difficult samples. These considerations, combined with OPERA's consistent performance improvements on diverse benchmarks, suggest that the efficiency-accuracy trade-off is favorable for practical clinical applications.

\section{Multiple-choice Evaluation Benchmark Details}

\subsection{Benchmark Source and Labels}
We extract the comprehensive fundus image evaluation benchmark from \cite{li2025visionunite}, which contains $845$ color fundus photographs annotated into $10$ disease categories (AMD, CSR, DR, DME, glaucoma, cataract, media haze, myopia, retinitis, and tessellation). The label distribution of this benchmark is reported in \cref{labeldis}.

\subsection{Multiple-choice Construction}
To cast this task as a multiple-choice problem, we create \emph{four} answer options for each image and question. The correct option is always the ground-truth disease label of that image. The remaining three options are distractors randomly sampled \emph{without replacement} from the other disease labels (i.e., excluding the ground-truth label), ensuring that no duplicate options appear within one question.

\begin{table*}[!ht]
  \centering
  \caption{The label distribution of the multiple-choice evaluation benchmark.}
  \vspace{-2mm}
    \begin{tabular}{c|c|c|c}
    \toprule
    \textbf{Label} & Number & \textbf{Label} & Number \\
    \midrule
    Age-related Macular Degeneration (AMD) & 171 & Cataract & 20\\
    Central Serous Retinopathy (CSR)& 7 & Diabetic Retinopathy (DR) & 149 \\
    Glaucoma & 162 & Media Haze & 31\\
    Myopia & 226 & Retinitis & 9 \\
    Diabetic Macular Edema (DME)& 58 & Tessellation & 12\\ \hline
    \multicolumn{2}{c|}{\textbf{Summary}} & \multicolumn{2}{c}{\textbf{845 images}}\\
    \bottomrule
    \end{tabular} 
  \label{labeldis}%
  \vspace{-2mm}
\end{table*}%

\section{Detailed Experimental Setup}
\label{sec:experimental_setup}

\subsection{Hardware and Software Environment}
All experiments are conducted on NVIDIA RTX A6000 GPUs (48GB memory). We implement our framework using PyTorch 1.13 with CUDA support. We leverage the \texttt{timm} library for model architectures and \texttt{scikit-learn} for evaluation metrics. To ensure reproducibility, we fix the random seed to 42 across all experiments and enable \texttt{cudnn.benchmark} for optimized GPU performance.

\subsection{Image Preprocessing and Hyperparameters}
Input images are resized to 224$\times$224 pixels as for the task of classification. During evaluation, we apply the standard preprocessing pipeline without data augmentation. For the validation-based weight tuning, we optionally employ RandAugment with magnitude of 6 and standard deviation of 0.5 (\texttt{rand-m6-mstd0.5-inc1}). Images are normalized using ImageNet statistics for consistency with pretrained model expectations. The related hyperparameters are as shown in \cref{tab:hyperparams}.

\begin{table}[!ht]
\centering
\caption{Hyperparameters used in the ensemble experiments for the task of classification. The weight tuning optimizer concerns only the cross-expert \emph{combination} weights over the \emph{frozen} pool.}
\vspace{-2mm}
\label{tab:hyperparams}
\begin{tabular}{lc}
\toprule
\textbf{Hyperparameter} & \textbf{Value} \\
\midrule
Batch size & 128 \\
Number of workers & 4 \\
Random seed & 42 \\
Weight tuning epochs & 20 \\
Weight tuning learning rate & 0.01 \\
Weight tuning optimizer & Adam \\
Early stopping patience & 5 \\
Initial confidence temperature & 2.0 \\
Weight softmax temperature & 10.0 \\
\bottomrule
\end{tabular}
\vspace{-4mm}
\end{table}

\subsection{Weight Tuning of Expert Profiling Module}
We employ a two-stage evaluation protocol:

\noindent\textbf{Stage 1: Weight Tuning on Validation Set.} Ensemble weights are optimized on the validation set using gradient descent. We use the Adam optimizer with a learning rate of 0.01 for 20 epochs with early stopping. The objective is to maximize the average AUC across all disease categories. Only the ensemble combination weights are trainable and all expert agent parameters remain frozen.

\noindent\textbf{Stage 2: Test Set Evaluation.} The tuned weights are applied to evaluate final performance on the held-out test set, ensuring unbiased assessment.

\subsection{Adaptive Weighting Strategies of Test-time Modules}
During inference, we incorporate several unsupervised adaptive weighting strategies:

\begin{itemize}
    \item \textbf{Prediction Agreement:} When models exhibit high agreement (agreement score $> 0.90$) and low variance ($< 0.03$), we rely on the base model-level weights.
    \item \textbf{Disagreement Handling:} When models disagree (agreement $< 0.60$ or variance $> 0.12$), we weight them by inverse entropy to trust more confident predictions.
    \item \textbf{Extremity Weighting:} For intermediate cases, models with predictions further from 0.5 receive higher weights, combined with base weight bias (70\% extremity, 30\% base).
    \item \textbf{Temperature Scaling:} Per-model temperature calibration adjusts overconfident ($T=1.5$ for sharpness $> 0.4$) or underconfident ($T=0.7$ for sharpness $< 0.1$) predictions.
\end{itemize}

\begin{figure*}[!t]
  \begin{center}
    \centerline{\includegraphics[width=\textwidth]{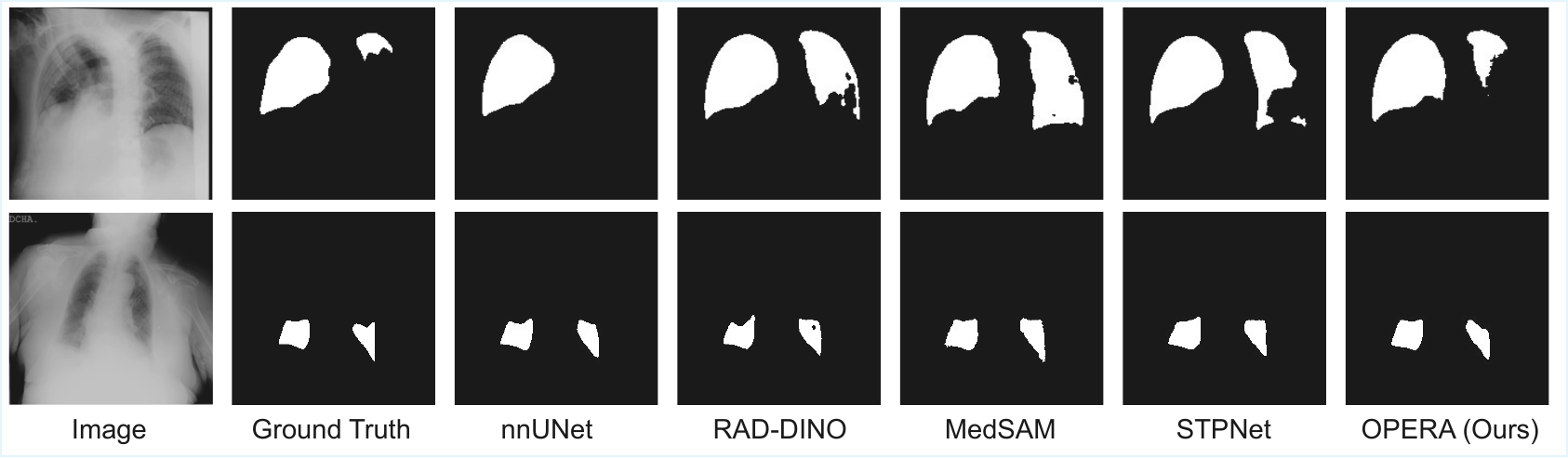}}
    \vspace{-4mm}
    \caption{Additional qualitative segmentation results on the COVID-Xray (QaTa-COV19) dataset. From left to right: original image, ground truth, nnUNet, RAD-DINO, MedSAM, STPNet, and OPERA (Ours).}
    \label{sup_results1}
  \end{center}
  \vspace{-5mm}
\end{figure*}

\begin{figure*}[!t]
  \begin{center}
    \centerline{\includegraphics[width=\textwidth]{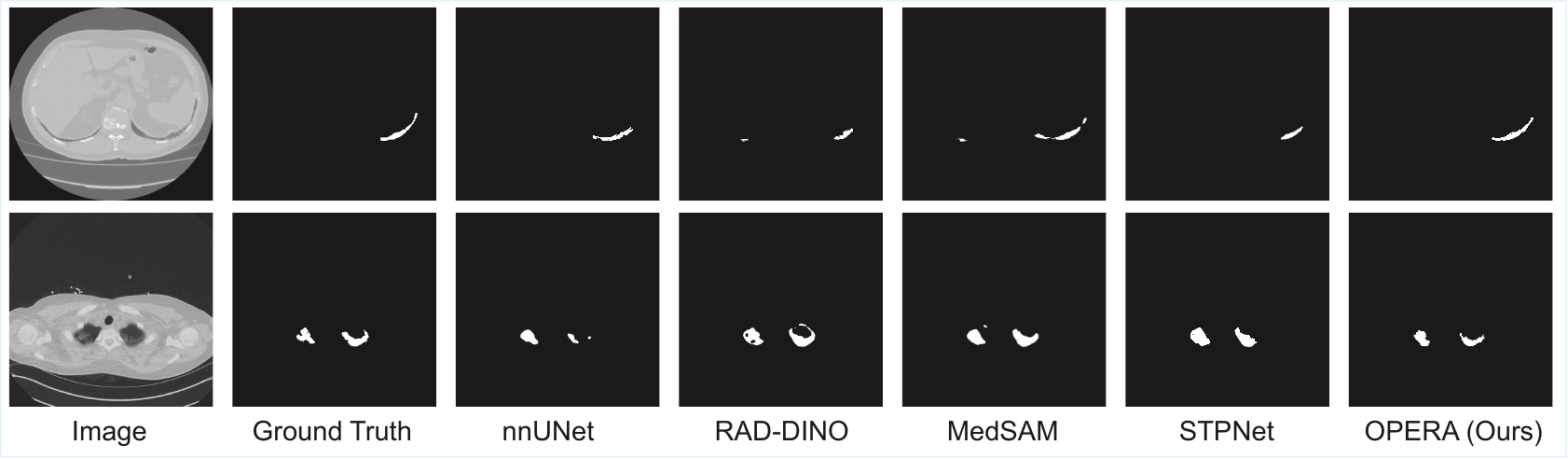}}
    \vspace{-4mm}
    \caption{Additional qualitative segmentation results on the COVID-CT (MosMedData+) dataset. From left to right: original image, ground truth, nnUNet, RAD-DINO, MedSAM, STPNet, and OPERA (Ours).}
    \label{sup_results2}
  \end{center}
  \vspace{-5mm}
\end{figure*}

\section{More Visualization Results}

\subsection{Visualization Results on COVID-Xray (QaTa-COV19)}

We provide additional qualitative segmentation results on the QaTa-COV19 dataset in Fig.~\ref{sup_results1}. The top row shows a case with infection segmentation where RAD-DINO and MedSAM produce noisy boundaries with scattered artifacts, while OPERA achieves smoother contours that closely match the ground truth. The bottom row depicts bilateral lung infections where OPERA consistently produces more accurate and complete segmentations with fewer false positives, demonstrating its robustness to the low-contrast and ambiguous boundaries commonly encountered in COVID-19 chest radiographs.

\subsection{Visualization Results on COVID-CT (MosMedData+)}

We present additional qualitative results on the MosMedData+ dataset in Fig.~\ref{sup_results2}. The top row shows a challenging case with a thin, elongated lesion structure where nnUNet under-segments the region and RAD-DINO produces discontinuous predictions. The bottom row presents a case with bilateral lesions of varying sizes, where MedSAM and RAD-DINO exhibit over-segmentation artifacts. In both cases, OPERA generates segmentations that more faithfully capture the shape and extent of the infection regions while effectively suppressing spurious predictions.

\subsection{Visualization Results on LA-MRI Dataset and Pancreas-CT Dataset}

We show additional qualitative results on the LA-MRI dataset and Pancreas-CT under the limited annotation setting (20\% labeled data) in Fig.~\ref{fig_results_2}. As shown in Fig.~\ref{fig_results_2} (top row), OPERA preserves the thin atrial boundaries and reduces local discontinuities better, which are common failure cases when pseudo-labels are noisy. The improvements highlight OPERA’s ability to maintain geometric consistency with limited annotations. Fig.~\ref{fig_results_2} (bottom row) shows that OPERA suppresses spurious foreground regions and recovers more complete pancreas structures, indicating improved robustness in ambiguous boundaries and noisy contexts in MRI and CT.

\begin{figure*}[!ht]
  \begin{center}
    \centerline{\includegraphics[width=\textwidth]{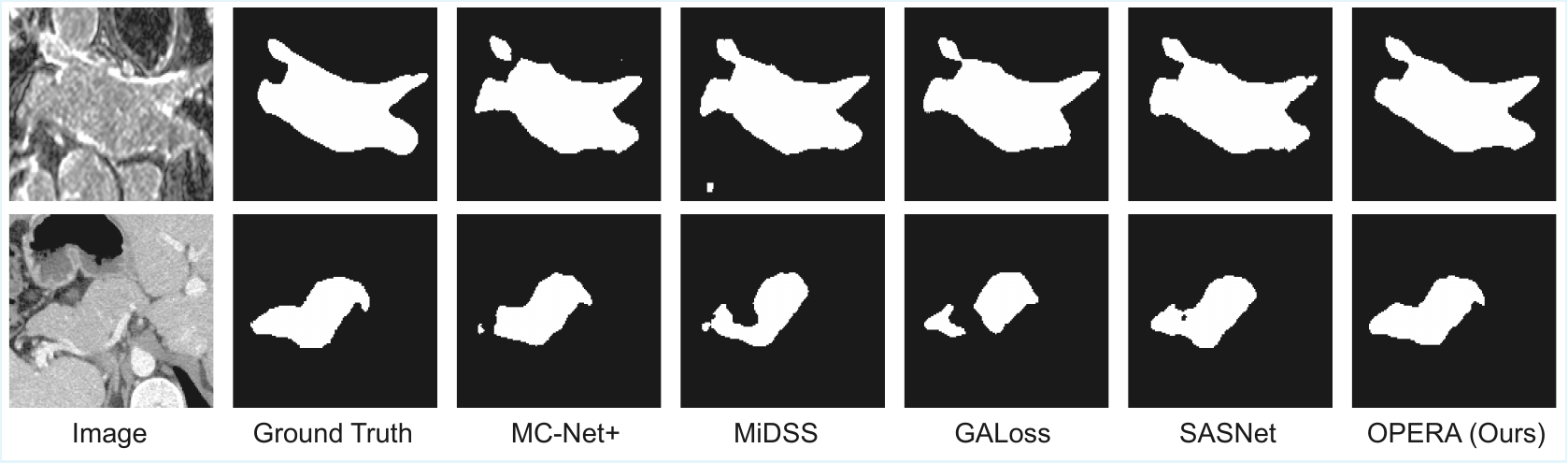}}
    \vspace{-4mm}
    \caption{Qualitative results under limited annotations (20\% labeled data). \textbf{Top:} Left atrial MRI (LA-MRI). \textbf{Bottom:} Pancreas CT (Pancreas-CT). OPERA produces more complete structures with fewer spurious regions, cleaner boundaries and better coverage.}
    \label{fig_results_2}
  \end{center}
  \vspace{-4mm}
\end{figure*}

\begin{table*}[!ht]
  \centering
  \caption{The performance of fair same-expert-pool comparison. All ensemble baselines use the identical expert pool.}
  \vspace{-2mm}
    \begin{tabular}{l ccc ccc cc}
    \toprule
    & \multicolumn{3}{c}{OIA-DDR (cls. Linear Probing)} & \multicolumn{3}{c}{OIA-DDR (cls. Fully Fine-tune)} & \multicolumn{2}{c}{QaTa-COV19 (seg.)}\\
    \cmidrule(lr){2-4}\cmidrule(lr){5-7}\cmidrule(lr){8-9}
    Method (same 3-expert pool) & AUC & ACC & mAP & AUC & ACC & mAP & Dice & Jaccard \\
    \midrule
    Probability averaging & 74.26 & 57.14 & 40.33 & 85.33 & 71.85 & 50.14 & 79.37 & 69.81 \\
    Temperature-scaled averaging  & 74.59 & 57.41 & 40.60 & 85.36 & 71.91 & 50.19 & 79.31 & 69.73 \\
    Bayesian model averaging  & 74.52 & 57.30 & 40.46 & 85.39 & 71.96 & 50.21 & 79.87 & 70.14 \\
    Validation-weighted averaging & 75.13 & 58.07 & 41.24 & 85.94 & 72.12& 51.06 & 79.99 & 70.38 \\
    Learned stacking  & 76.71 & 59.46 & 42.93 & 86.08 & 72.20 & 51.29 & 80.26 & 70.59 \\
    Learnable MoE gating & 78.42 & 61.11 & 44.16 & 87.11 & 72.44 & 53.82 & 80.77 & 71.40 \\
    Test-Time Ensemble   & 79.78 & 62.36 & 45.51 & 86.74 & 72.25 & 52.68 & 81.03 & 71.82 \\
    \textbf{OPERA (Ours)} & \textbf{84.95} & \textbf{69.57} & \textbf{49.81} & \textbf{88.06} & \textbf{72.89} & \textbf{55.90} & \textbf{82.12} & \textbf{73.26} \\
    \bottomrule
    \end{tabular}
  \label{ResultsSame}%
  \vspace{-2mm}
\end{table*}%

\section{Comparison of the Same-expert-pool}
We conduct the comparison experiments of the same-expert-pool models as shown in \cref{ResultsSame}, which places every competing method (temperature-scaled / validation-weighted averaging, stacking, MoE gating, BMA, TTA) on the same frozen three-expert pool OPERA uses, so any retained margin cannot be attributed to ``simply using more models.'' On both OIA-DDR (cls.) and QaTa-COV19 (seg.), OPERA still beats the strongest same-pool baseline (Test-Time Ensemble) by \textbf{+5.17\% AUC} and \textbf{+1.09\% Dice}, localizing gains to the proposed routing/adaptation. Our method also has matched-budget cost that equals a standard ensemble including peak GPU memory and per-sample latency.

\section{Discussion of Fully Unsupervised Mode}
For scenarios where \textbf{no labeled target-domain data is available}, OPERA provides two viable alternatives. First, \textit{source-domain transfer}: EPM weights computed on a source-domain validation set transfer effectively to related target domains, as the relative expertise of expert agents often generalizes across datasets within the same modality. Second, \textit{fully unsupervised operation}: OPERA can operate entirely without labeled data by initializing with uniform model weights ($w_{\text{model},m} = 1/M$) and uniform class attention ($W_{\text{class},m,c} = 1/M$), relying solely on test-time components (PECC, DAA, IER) for adaptive fusion and adjustment.

To validate this unsupervised operational mode, we evaluate OPERA*, which bypasses the validation-based EPM initialization and uses uniform weights instead. As shown in Table \ref{tab:performance_oiaddr}, OPERA* achieves 83.18\% AUC on OIA-DDR under linear probing and 86.30\% under full fine-tuning, substantially outperforming all individual expert agents and remaining competitive with state-of-the-art methods such as RETFound (85.96\%) and FLAIR (85.54\%). Similarly, Table \ref{tab:performance_covid} demonstrates that OPERA* achieves 81.54\% Dice on QaTa-COV19 and 76.80\% Dice on MosMedData+, exceeding strong baselines including nnUNet, MedSAM, and RAD-DINO. While OPERA* shows modest performance reduction compared to the full OPERA framework (e.g., 86.30\% vs. 88.06\% AUC on OIA-DDR with fine-tuning), the gap is relatively small, confirming that the test-time adaptation modules (PECC, DAA, IER) contribute meaningful gains even without any labeled data. Specifically, Per-Expert Confidence Calibration (PECC) and Instance-level Expert Routing (IER) operate purely on prediction statistics, requiring no ground truth.

Compared to existing approaches, OPERA's validation requirement is fundamentally less demanding: foundation models such as RETFound~\cite{zhou2023foundation} and VisionFM~\cite{qiu2024development} require millions of samples for pretraining and often task-specific fine-tuning; domain adaptation methods typically require target-domain samples and retraining cycles; and traditional ensemble methods often require joint optimization across all experts with full training data. OPERA's one-time validation-based calibration represents minimal overhead that can leverage even small labeled subsets or existing annotated datasets from similar domains. Once the model weights are computed offline with Expert Profiling Module (EPM), OPERA can be deployed to new test distributions without any model updates, making it well-suited for clinical settings where continuous retraining is infeasible. For deployment scenarios where even small validation sets are unavailable, OPERA* provides a fully unsupervised alternative that still achieves competitive performance through test-time adaptation alone.

\section{More Discussion}
OPERA demonstrates that an offline-calibrated zero-retraining method can achieve state-of-the-art performance across a remarkably broad range of biomedical imaging tasks, rivaling or surpassing specialized models in each domain. Unlike contemporary medical foundation models that rely on extensive self-supervised pretraining and task-specific fine-tuning \cite{moor2023, yao2025, zhou2023foundation}, OPERA integrates multiple complementary experts solely at inference time, without any additional training. This design directly addresses the challenge of distribution shift and unseen clinical domains \cite{moor2023}, leveraging expert diversity to ensure robust generalization across fundus photography, chest X-rays, CT scans, and structured multimodal diagnostic queries. The offline-calibrated zero-retraining method marks a clear departure from monolithic foundation and vision-language models, such as InternVL \cite{chen2024internvl}, Qwen-VL \cite{bai2023qwen}, Mini-Gemini \cite{li2024mini}, and instruction-tuned architectures like InstructBLIP \cite{instructblip} and LLaVA-Med \cite{li2024llava}, which typically require large-scale aligned image-text datasets and substantial computational resources for fine-tuning or instruction tuning \cite{liu2024visual}. In contrast, OPERA retains the accuracy and robustness benefits of ensembling while eliminating retraining costs, offering a lightweight yet powerful alternative.

Across benchmark datasets, OPERA consistently outperforms contemporary single-model approaches. In retinal disease classification, OPERA achieves a new state-of-the-art on the RFMiD challenge, reaching 89.47\% AUC and 58.93\% mAP, improving upon the prior best method RET-CLIP \cite{du2024ret} by 3.35\% AUC and 7.66\% mAP. Notably, OPERA also exceeds the performance of ophthalmology-specific foundation models, including VisionFM \cite{qiu2024development}, FLAIR \cite{silva2025foundation}, and RETFound \cite{zhou2023foundation}, despite those models being explicitly optimized for fundus imaging. On the OIA-DDR benchmark, OPERA similarly outperforms recent multimodal frameworks such as UniMed-CLIP \cite{khattak2024unimed} and MM-Retinal \cite{wu2024mm}. These results suggest that fusing heterogeneous expert representations can yield stronger and more transferable features than any individual foundation model, even within highly specialized domains. Beyond ophthalmology, OPERA generalizes effectively to chest and abdominal imaging tasks. On the Chest X-Ray14 benchmark, OPERA achieves the highest AUC (83.05\%) and mAP (30.58\%), surpassing recent chest X-ray foundation models, including Ark+ \cite{ma2025fully} and global-local integration models \cite{yang2025chest}. OPERA also outperforms an ensemble-of-foundation-models baseline such as ELF~\cite{luo2025ensemble}. In abdominal CT organ classification (OrganSMNIST), OPERA achieves 98.8\% AUC and 84.3\% accuracy without any task-specific adaptation. In contrast, most existing vision foundation models remain domain-constrained such as CheXFound for chest radiographs or VisionFM for ophthalmology and both of them require retraining when transferred to new tasks \cite{yao2025, qiu2024development}. OPERA’s strong performance across disparate classification tasks highlights the importance of ensemble learning \cite{Kuncheva2003, Ju2017}.

OPERA’s advantages extend to segmentation tasks, despite not being explicitly trained for segmentation. On COVID-19 lesion segmentation benchmarks (QaTa-COV19 and MosMedData+), OPERA achieves Dice scores of 82.1\% and 77.0\%, outperforming fully supervised architectures such as nnU-Net \cite{isensee2021nnu} and recent prompt-based methods like STPNet \cite{shan2025stpnet}. These results are particularly notable given that segmentation-specific models often rely on carefully optimized architectures such as Transformer-based variant \cite{chen2024transunet}. We attribute OPERA’s gains to its dynamic weighting strategies (EPM, PECC, DAA, IER), which enable per-sample adaptation by modulating expert contributions based on confidence, disagreement, and class relevance. Furthermore, OPERA demonstrates strong robustness in data-scarce situations. With only 20\% labeled data, it outperforms recent semi-supervised segmentation methods, including mutual consistency learning \cite{wu2022mutual}, pseudo-label guided contrastive learning \cite{basak2023pseudo}, and anatomically-aware uncertainty models \cite{adiga2024anatomically}. While these approaches leverage unlabeled data through additional training, OPERA achieves superior Dice scores without any retraining, further validating the effectiveness of inference-time ensemble adaptation. These findings align with prior work \cite{Noothout2022, Codella2016} showing that well-constructed ensembles can outperform single models even in limited-data settings.

Despite OPERA achieving broad generalization across diverse biomedical imaging tasks, it inherits several limitations commonly associated with ensemble-based methods. Running multiple expert agents in parallel inevitably increases inference cost and system latency \cite{Noothout2022}. We also want to clarify the scope of OPERA's ``zero-retraining'' design. The term specifically refers to the \textit{deployment phase}: OPERA requires \textbf{no backpropagation, no gradient computation, and no parameter updates to expert agents during inference}. This distinguishes OPERA from conventional ensemble methods that necessitate joint retraining, iterative fine-tuning, or complex meta-learning methods. However, OPERA's Expert Profiling Module (EPM) does require a labeled validation set to initialize model-level weights $\mathbf{w}_{\text{model}}$, class-level attention weights $\mathbf{W}_{\text{class}}$, and temperature parameters $\{T_m\}_{m=1}^{M}$ during an offline calibration phase. This calibration is a \textit{one-time, lightweight computation} that: (1) requires only forward passes through expert agents without backpropagation; (2) computes AUC scores analytically without iterative optimization; (3) completes in approximately five minutes on standard validation splits such as for the Chest X-Ray14 dataset~\cite{wang2017chestx}; and (4) can be performed once and cached for subsequent deployment. At the same time, we also provide the fully unsupervised mode of OPERA and its performance still remains competitive as shown in \cref{tab:performance_oiaddr} and \cref{tab:performance_covid}.

\end{document}